\documentclass[11pt]{article}

\usepackage[margin=1in]{geometry}

\usepackage[round,authoryear]{natbib}

\usepackage[
    colorlinks=true,
    citecolor=blue,
    linkcolor=blue,
    urlcolor=blue
]{hyperref}
\usepackage{booktabs}
\usepackage{tabularx}
\usepackage{array}
\usepackage{graphicx}
\usepackage{float}
\usepackage{amsmath}
\usepackage{amssymb}
\usepackage{amsthm}
\usepackage{bm}
\usepackage[normalem]{ulem}
\usepackage{graphicx}
\usepackage{algorithm}
\usepackage{algpseudocode}
\theoremstyle{definition}
\newtheorem{example}{Example}[section]
\newtheorem{theorem}{Theorem}[section]
\newcommand{\noi}{\noindent}
\newcommand{\bmu}{\mbox{\boldmath $\mu$}}
\newcommand{\bsigma}{\mbox{\boldmath $\sigma$}}
\newcommand{\btheta}{\mbox{\boldmath $\theta$}}
\newcommand{\bvartheta}{\mbox{\boldmath $\vartheta$}}

\newcommand{\bsvartheta}{\mbox{\boldmath \small $\vartheta$}}
\newcommand{\bdelta}{\mbox{\boldmath $\delta$}}
\newcommand{\brho}{\mbox{\boldmath $\rho$}}
\newcommand{\balpha}{\mbox{\boldmath $\alpha$}}
\newcommand{\bkappa}{\mbox{\boldmath $\kappa$}}

\newcommand{\bSigma}{\mbox{\boldmath $\Sigma$}}

\begin{document}

\title{QD-PINN:A Physics-Informed Neural Network Approach to Bayesian Quickest Detection in High-Dimensional Diffusions}

\author{
Hang Ye
\qquad
Yi Liu
\qquad
Haipeng Xing\thanks{Corresponding author: \texttt{haipeng.xing@stonybrook.edu}}
\\[0.6em]
\small Department of Applied Mathematics and Statistics, Stony Brook University
}

\date{}

\maketitle

\begin{abstract}
Quickest detection problems (QDPs) arise in a wide range of applications. However, the complementarity 
structure of the associated variational inequalities poses substantial challenges for traditional numerical methods, 
particularly in high-dimensional settings. We propose a PINN-based framework for estimating the optimal stopping boundary 
in QDPs with high-dimensional diffusion processes and spatial dependence, 
while overcoming the curse of dimensionality. Theoretically, we establish a novel equivalence between the free-boundary PDE 
and an infinite-dimensional stochastic optimization problem, rigorously justifying the use of finite-dimensional neural networks 
to approximate the true solution. Furthermore, we assess the robustness of the proposed framework by quantifying 
uncertainty in the learned stopping boundaries using deep ensembles and Monte Carlo dropout. The results show 
low variability across independent training runs and stochastic forward passes, indicating stable and consistent 
boundary estimates. We also compare our PINN results with reference numerical solutions where available and find
close agreement.
\end{abstract}

\section{Introduction}
\label{sec:introduction}

Sequential change-point detection, or quickest detection, aims 
to identify distribution shifts in observed stochastic processes 
as quickly as possible while controlling false alarms. Driven by 
the growth of high-dimensional, high-frequency, and spatially 
correlated streaming data, this task is vital across diverse 
domains, including financial anomaly detection, cybersecurity, 
industrial monitoring, healthcare, social media, and the Internet 
of Things (IoT) \citep{CohenZhao2015, tasoulis2018, 
Leevy2020, Xie2021}. However, extending classical 
univariate techniques to multivariate streams remains 
challenging. When observations are highly dimensional 
and high-frequency, real-time inference becomes difficult. 
This challenge is further exacerbated by intricate spatial 
correlations and cross-channel dependencies, which 
increase computational costs and make traditional 
methods impractical.

In the Bayesian formulation pioneered by Shiryaev 
\citep{Shiryaev1963, Shiryaev1964, Shiryaev1967}, 
quickest detection is cast as an optimal stopping 
problem. For diffusion processes, the strong Markov 
property yields an analytical characterization through 
a partial differential equation (PDE) with a free boundary 
\citep{Poor1998, Beibel2000, Gapeev2006, Peskir2006, 
Bayraktar2005, Gapeev2013}. Spatial correlation in 
high-dimensional data is explicitly captured by the 
cross-derivative terms of the diffusion operator. The 
free boundary partitions the state space into 
continuation and stopping regions. Consequently, 
the value function and optimal stopping boundary 
must be determined jointly, subject to the governing 
PDE, boundary conditions, and the value-matching, 
smooth-fit, and complementarity constraints of the 
associated variational inequality.

However, solving high-dimensional free-boundary 
problems (FBPs) with dense spatial coupling is 
difficult using conventional numerical methods. 
Grid-based schemes suffer from the 
curse of dimensionality---particularly when 
cross-derivatives prevent equation decoupling---while 
standard Monte Carlo methods cannot directly resolve 
the geometry of unknown stopping boundaries. 
This limitation motivates mesh-free paradigms 
that simultaneously approximate the value function 
and free boundary in high-dimensional state spaces. 
Physics-informed neural networks (PINNs) offer a 
natural framework for this challenge 
\citep{Raissi2019, Raissi2020, Wang2021Stefan}. 
By representing unknown functions with neural 
networks and enforcing governing equations and 
boundary conditions through customized loss 
functions, their mesh-free formulation is particularly 
effective when both the solution and free boundary are unknown.

We develop a PINN framework for sequential 
change-point detection in high-dimensional diffusion 
processes by formulating the task as a 
continuous-time optimal stopping problem 
governed by a high-dimensional FBP. 
Working directly with the original state variables, we 
use the Kolmogorov transformation 
\citep{Kolmogorov1992} to derive a smooth-fit 
condition without requiring an additional reduced 
state representation. We then reformulate this 
FBP as an infinite-dimensional 
stochastic optimization problem, approximating 
its solution via two neural networks that jointly 
estimate the value function and the stopping 
boundary. The network training is guided by a 
comprehensive loss function enforcing the 
governing PDE, boundary conditions, value-matching, 
smooth-fit, and complementarity constraints, while 
specialized high-dimensional sampling strategies are 
introduced to significantly improve approximation 
accuracy near the free boundary.
We demonstrate the effectiveness of our approach by 
estimating optimal stopping boundaries in change-point 
detection problems for high-dimensional Brownian 
motion and Ornstein--Uhlenbeck processes. 
Experiments consider diffusion dimensions up to 
$1000$ and several spatial correlation structures, 
such as tridiagonal, and Toeplitz correlation matrices. 
We also analyze sources of 
uncertainty in our PINN method and propose 
approaches for quantifying this uncertainty.

Our specific contributions are summarized as follows:
(i) We establish a rigorous theoretical equivalence 
between high-dimensional quickest detection and 
deep learning approximation. Bridged by an infinite-dimensional 
stochastic optimization framework, we provides 
mathematical justification for using finite-dimensional 
neural network architectures to approximate 
continuous-time optimal stopping solutions.
(ii) We develop a dual-network PINN framework that 
jointly approximates the value function and 
stopping boundary. A unified loss function embeds 
the governing PDE, boundary conditions, 
value-matching, smooth-fit, and complementarity 
constraints. The architecture also employs 
high-dimensional sampling to promote accurate 
convergence near the stopping boundary.

The remainder of the paper is organized as follows. 
Section~\ref{sec:related_work} reviews related work 
on deep learning for optimal stopping and FBP. 
Section~\ref{sec:seq_cp_detect} presents the 
Bayesian formulation of sequential change-point 
detection and its optimal stopping representation. 
Section~\ref{sec:method} develops the proposed 
PINN methodology and assesses uncertainty in 
the resulting solutions. Section~\ref{sec:experiments} 
presents the numerical experiments, and 
Section~\ref{sec:conclusion} concludes the paper.

\section{Literature Review}
\label{sec:related_work}
Deep-learning approaches to optimal stopping can be broadly divided into simulation-based and PDE-based methods. Simulation-based approaches learn stopping decisions or continuation values directly from simulated trajectories. For example, \citet{becker2019deep} introduced \emph{deep optimal stopping}, which learns stopping strategies from Monte Carlo samples and extends to continuous-time problems through time discretization. Recently, \citet{herrera2024optimal} developed randomized neural-network methods that train only the output layer, reducing optimization to a linear regression that scales well to high-dimensional applications.
A complementary paradigm uses PINNs to solve FBPs directly without relying on sample paths. \citet{Wang2021Stefan} developed a multi-network PINN framework for Stefan problems by jointly approximating the PDE solution and moving boundary. For optimal stopping, \citet{Gatta2023} proposed a meshless PINN method for multi-asset American options that jointly approximates the value function and exercise boundary while enforcing PDE and free-boundary constraints.

Our framework follows this PDE-based paradigm but focuses on the high-dimensional FBP arising in quickest change-point detection. While sharing the dual-network structure of existing moving-boundary methods, our approach introduces several key distinctions. We employ the Kolmogorov transformation to enforce smooth-fit conditions directly in the original state space without dimension reduction, incorporate strict complementarity constraints into a unified loss objective, and quantify uncertainty in the PINN approximations. We further introduce high-dimensional sampling strategies, demonstrating empirical scalability to dimensions well beyond those addressed by existing meshless PDE methods.

\section{Sequential change-point detection}
\label{sec:seq_cp_detect}

This section presents the sequential change-point 
detection problem for the drift coefficient of an 
observable diffusion process over an infinite horizon. 
Formulated within a Bayesian framework, this detection 
task is recast as a continuous-time optimal stopping 
problem with free-boundary constraints.

\subsection{Bayesian formulation}
\label{subsec:bayesian_formulation}

Consider the Bayesian quickest detection problem for an 
$n$-dimensional Wiener process. Let \(\theta\) be a 
nonnegative random variable such that $P_{\pi}
(\theta=0)=\pi$ and $P_{\pi}(\theta>t\mid\theta>0)
=e^{-\lambda t}$ for all $t>0$ and some \(\lambda>0\). 
Let \(\widetilde{\bf w}=\{\widetilde{\bf w}_t
\}_{t\ge0}\) be an \(n\)-dimensional Wiener process with 
correlation matrix \(\boldsymbol\rho\), starting from 
\(\mathbf 0\), and assume that \(\theta\) and 
\(\widetilde{\bf w}\) are independent. For 
convenience, represent the correlated Wiener process 
as $\widetilde{\bf w}_t=\brho^{1/2} {\bf w}_t$, 
where ${\bf w}_t$ is an  \(n\)-dimensional standard 
Wiener process and $\brho^{1/2}$ is an \(n\times n\) 
upper-triangular matrix satisfying $\brho=\brho^{1/2} 
(\brho^{1/2})'$, with $(\brho^{1/2})'$ denoting the 
transpose of \(\boldsymbol\rho^{1/2}\).
Suppose that one observes a $n$-dimensional 
continuous process $\{{ \bf x}_t\}$ solving the SDE
\begin{equation}\label{eq:cp_sde}
d{\bf x}_t = \left\{ \begin{array}{ll}
\bmu_0 ({\bf x}_t) dt + \bsigma({\bf x}_t)
d\widetilde{\bf w}_t, & t \in (0, \theta], \\
\bmu_1({\bf x}_t) dt + \bsigma( {\bf x}_t) 
d\widetilde{\bf w}_t, & t \in (\theta, \infty),
\end{array}\right.
\end{equation}

\noi with ${\bf x}_0 \in \mathbb{R}^n$,
where $\bmu_i: \mathbb{R}^n \to \mathbb{R}^n$,
$i=0,1,$ and $\bsigma: \mathbb{R}^n \to
\mathbb{R}^{n \times n}$ are 
continuously differentiable functions.
Let $\bdelta (\mathbf{x}) = \bmu_1({\bf x})
-\bmu_0(\mathbf{x})$, $\widetilde{\bsigma}
(\mathbf{x}) = \bsigma(\mathbf{x}) \brho^{1/2}$,
and $\bSigma(\mathbf{x}):=
\widetilde{\bsigma}(\mathbf{x})
\widetilde{\bsigma}'(\mathbf{x})$.
We assume that, for all ${\bf x} \in \mathbb{R}^n$, 
$\bSigma(\mathbf{x})$ are positive definite
$n\times n$ matrices, and there exist some fixed $\varrho_1,
\varrho_2>0$ such that,
$\varrho(\mathbf{x}) := \bdelta (\mathbf{x})'
\boldsymbol{\Sigma}^{-1}(\mathbf{x})
 \bdelta (\mathbf{x}) \in (0,\varrho_1)$ and
$\left\| \bmu_i(\mathbf{x}) \right\| +
\left\| \widetilde{\bsigma}(\mathbf{x}) \right\|
\leq \varrho_2 \left( 1+\|\mathbf{x}\| \right)$.
Then it follows from \cite[][Theorem 4.6]{Liptser2001}
that \eqref{eq:cp_sde} admits a unique strong 
solution under $\theta=s$ for some fixed $s$. 
Therefore, $P_{\pi}( {\bf x}\in\cdot | \theta=s)
= P^s ( {\bf x}\in\cdot )$ is the 
distribution law of a time-homogeneous diffusion 
process started at some $\mathbf{x}$,
with covariance matrix $\bsigma({\bf x})$
and the drift rate changing from $\bmu_0({\bf x})$
to $\bmu_1({\bf x})$ at time $s\in[0,\infty)$.
Then the probability measure $P_{\pi}$ has the structure
$P_{\pi} = \pi P^0 + (1-\pi) \int_0^\infty P^s 
\lambda e^{-\lambda s}ds$ for any fixed $\pi\in[0,1)$.
Given continuous observations of ${\bf x}_t$, we seek 
an optimal stopping time \(\tau\) that minimizes the 
Bayesian risk
\begin{equation}\label{eq:bayesian_risk}
\inf_{\tau} \big( P_{\pi}(\tau<\theta)
+ E_{\pi} \big[ F(\tau-\theta)\mathbf{1}_{\{\tau\geq\theta\}}
\big] \big).
\end{equation}
\noi Here, the first term represents the false-alarm 
probability, while the second is the expected delay 
cost. The delay penalty function satisfies \(F(t)\geq0\) 
for \(t\geq0\) and \(F(t)=0\) for \(t\leq0\).

Let $\pi_t = P_{\pi}( \theta\leq t | \mathcal{F}_t^{
\bf x})$ be the posterior probability process with 
$P_{\pi}(\pi_0=\pi)=1$. By Bayes' theorem,
$$
\begin{aligned}
&\pi_t = \pi \frac{dP^0}{dP_{\pi}} (t, {\bf x}) + (1-\pi)
\int_0^t \frac{dP^s}{dP_{\pi}}(t, {\bf x})
\lambda e^{-\lambda s}ds, \\
&1-\pi_t = (1-\pi)e^{-\lambda t}
\frac{dP^t}{dP_{\pi}}(t, {\bf x}),
\end{aligned}
$$

\noi  where $dP^s/dP_{\pi}$ denotes the Radon--Nikodym 
derivative of $(P^s | \mathcal{F}_t^{\bf x})$ with respect 
to $(P_\pi | \mathcal{F}_t^{\bf x})$. Let $L_t = \exp 
\big( \int_0^t \bdelta({\bf x}_s)' 
\bSigma^{-1}({\bf x}_s) \big[ d\mathbf{x}_s 
-\frac{1}{2} (\bmu_1({\bf x}_s) + \bmu_0({\bf x}_s) )ds
\big] \big)$ be the likelihood ratio at $t$.
Define the innovation 
process $\overline{\bf w} = \{\overline{\bf w}_t
\}_{t\geq 0} $ such that $d\overline{\bf w}_t =
\widetilde{\bsigma}^{-1}({\bf x}_t) [  d {\bf x}_t-
\left( \bmu_0({\bf x}_t) + \pi_t \bdelta({\bf x}_t) 
\right)dt ]$, which is a standard Wiener 
process under the measure $P_{\pi}$ with respect 
to the filtration $\{\mathcal{F}_t\}_{t\geq 0}$.
Then the process $({\bf x}_t, \pi_t)$ is a unique solution 
to the following SDEs:
\begin{subequations}
\begin{equation}\label{eq:posterior_x}
d {\bf x}_t = \left[ \bmu_0({\bf x}_t) + \pi_t 
\bdelta( {\bf x}_t) \right]dt + \widetilde{\bsigma} ({\bf x}_t)
d\overline{\bf w}_t, 
\end{equation}
\begin{equation}\label{eq:posterior_pi}
d\pi_t = \lambda (1-\pi_t) dt + \pi_t (1-\pi_t)
\bdelta({\bf x}_t)' \widetilde{\bsigma}^{-1}
({\bf x}_t) d\overline{\bf w}_t.
\end{equation}
\end{subequations}

\subsection{Optimal stopping and free-boundary
}\label{subsec:opt_stop}

To solve the Bayesian risk minimization problem 
\eqref{eq:bayesian_risk}, 
we consider a more general setup. 
Suppose that $r$ is the discount rate and 
$F(\pi_s, {\bf x}_s)$ is the running cost at time $s$. 
Assume that $F(\pi_s, {\bf x}_s)$ is bounded and 
continuously differentiable. Then given ${\bf x}_0={\bf x}$, 
the total loss due to a cost of the observations 
prior to stopping time $\tau$ is expressed as
$E_{\pi, {\bf x}} \big( \int_0^\tau e^{-rs}
F(\pi_s, {\bf x}_s) ds \big)$, where 
$P_{\pi,\mathbf{x}}$ is a measure of the diffusion process 
$\{(\pi_t, {\bf x}_t)\}$ started at
$(\pi,\mathbf{x}) \in (0,1)\times\mathbb{R}^n$
and solving SDEs \eqref{eq:posterior_x} 
and \eqref{eq:posterior_pi}.
After stopping the observations at time $\tau$, 
the loss due to a false alarm at time $\tau$ is given by
$G(\pi_\tau, {\bf x}_\tau )$. 
Given that ${\bf x}_0={\bf x}$ and $\pi_0=\pi$, 
consider the following optimal stopping problem
\begin{equation}
\inf_{\tau} E_{\pi, {\bf x}} \big(
\int_0^\tau e^{-rs} F(\pi_s, {\bf x}_s)ds +
e^{-r\tau}G(\pi_\tau, {\bf x}_{\tau} ) \big),
\label{eq:general_optimal_stopping}
\end{equation}

\noi  where $\tau$ is a stopping time of the Markov process 
$\{(\pi_t,{\bf x}_t)\}$.
Setting \(r=0\), \(F(\pi,\mathbf{x})=c\pi\), and
$G(\pi, {\bf x})=1-\pi$ yields the 
classical Bayesian risk $P_{\pi}(\tau<\theta)+
cE_{\pi}(\tau-\theta)^+$, which combines the 
probability of false alarm with the expected 
delay cost, as in the sequential 
change-point literature \citep{Shiryaev1967}.

Let $V(\pi, {\bf x})$ be the value function of the optimal 
stopping problem \eqref{eq:general_optimal_stopping}.
Then one can show that $V(\pi, {\bf x})$ satisfies
the following variational inequality,
\begin{equation}\label{eq:variational_inequality}
\begin{aligned}
& \min \left\{ \mathcal{L}_{\pi, {\bf x}} V(\pi,\mathbf{x}) - 
rV(\pi,\mathbf{x}) + F(\pi,\mathbf{x}), \right. \\
& \qquad \qquad \left. G(\pi,\mathbf{x}) - V(\pi,\mathbf{x}) 
\right\} = 0,
\end{aligned}
\end{equation}

\noi in which the operator $\mathcal{L}_{\pi, {\bf x}}$
is defined as
$$
\begin{aligned}
& \mathcal{L}_{\pi, {\bf x}} = \frac{\varrho({\bf x})}{2}
\pi^2(1-\pi)^2 \frac{\partial^2}{\partial\pi^2} + \pi(1-\pi)
\bdelta' \frac{\partial^2} {\partial\pi\, 
\partial\mathbf{x}}\\
& + \lambda(1-\pi) \frac{\partial}{\partial\pi}
+\frac{1}{2} \operatorname{tr} \Big(
\bSigma \frac{\partial^2}
{\partial\mathbf{x}\,\partial\mathbf{x}'}
\Big) + \left( \bmu_0 + \pi \bdelta
\right)' \frac{\partial}{\partial {\bf x}},
\end{aligned}
$$

\noi Following the arguments in \citet[Section 5.2]{Pham2009}, 
$V(\pi, {\bf x})$ is the unique viscosity solution to 
\eqref{eq:variational_inequality}. 
Making use of optimal stopping theory, the optimal 
stopping time in \eqref{eq:general_optimal_stopping} 
is given by 
$$
\tau_* = \inf \left\{ t\geq 0 | V(\pi_t, {\bf x}_t) =
G(\pi_t, {\bf x}_t) \right\}$$ 
and the continuation region 
has the form 
$$
\mathcal{C} = \left\{ (\pi, {\bf x}) \in (0,1)\times
\mathbb{R}^n \mid V(\pi,\mathbf{x})<G(\pi, {\bf x}) \right\}.
$$
The theory further implies that there exists a 
function $S({\bf x}) \in (0, 1)$ such that the continuation 
region $\mathcal{C}$ takes the form
$\mathcal{C} = \left\{ (\pi,{\bf x}) \in (0,1)\times
\mathbb{R}^n \mid \pi < S({\bf x}) \right\}$,
the stopping region is the closure of the set
$\mathcal{S} = \left\{ (\pi, {\bf x}) \in (0,1)\times
\mathbb{R}^n \mid \pi>S(\mathbf{x}) \right\}$,
and the optimal decision rule in  
\eqref{eq:general_optimal_stopping} has the structure
$\tau_* = \inf \left\{ t\geq 0 \mid \pi_t\geq S({\bf x}_t) 
\right\}$. According to \citet[Section IV.8]{Peskir2006}, 
the variational inequality \eqref{eq:variational_inequality} 
can be represented by its associated FBP,
\begin{subequations}\label{eq:fb_system}
\begin{eqnarray} 
&& \mathcal{L}_{\pi, {\bf x}} V(\pi, {\bf x}) - rV(\pi,{\bf x})
+ F(\pi, {\bf x}) = 0 \mbox{ on } \mathcal{C}, \qquad
\label{eq:fb_pde}\\
&& V(\pi, {\bf x}) \big |_{\pi=S({\bf x})^-} = G(\pi, {\bf x})
\big |_{\pi=S({\bf x})^-}, \label{eq:fb_value_matching} \\
&& V(\pi, {\bf x}) = G(\pi, {\bf x} ) \quad \mbox{on }
\mathcal{S}, \label{eq:fb_stopping}\\
&& V(\pi,\mathbf{x}) < G(\pi,\mathbf{x}) \quad 
\mbox{on } \mathcal{C}. \label{eq:fb_inequality}
\end{eqnarray}

\noi The FBP \eqref{eq:fb_pde}--\eqref{eq:fb_inequality} 
admits multiple solutions. To find a unique solution
of the problem, we use the Kolmogorov transformation 
to obtain the following smooth fit condition
\citep{Kolmogorov1992},
\begin{equation}\label{eq:fb_smfit}
\frac{\partial V}{\partial\pi} + \frac{ \bdelta({\bf x})'
\big( \partial V/\partial {\bf x} \big)}{\pi(1-\pi)
\varrho(\mathbf{x})} \Big|_{\pi=S({\bf x})^-}
= \frac{\partial G}{\partial\pi}
\Big|_{\pi=S({\bf x})^-}.
\end{equation}
\end{subequations}

We then need to solve the FBP
\eqref{eq:fb_pde}--\eqref{eq:fb_smfit} such that the 
resulting boundary is continuous and of bounded variation.

\section{Physical Informed Neural Network}
\label{sec:method}

This section first establishes the equivalence between the 
FBP \eqref{eq:fb_pde}--\eqref{eq:fb_smfit} and an 
infinite-dimensional stochastic optimization framework, 
and then proposes a dual-network PINN architecture 
driven by a unified loss function that embeds the 
governing PDE, boundary conditions, value-matching, 
smooth-fit, and complementarity constraints. Finally, it 
analyzes the sources of approximation uncertainty.

\subsection{Infinite dimensional stochastic optimization}

Let $(\Omega,\mathcal{P},\mathbb{P})$ be a probability 
space. Let $\mathcal{I}:\Omega\to[0,1]$ and
$\mathcal{X}:\Omega\to\mathbb{R}^n$
be independent random variables. Define the 
following residuals corresponding to equations
\eqref{eq:fb_pde}, \eqref{eq:fb_value_matching},
\eqref{eq:fb_stopping}, and \eqref{eq:fb_smfit},
respectively,
\begin{align*}
& \mathcal{R}_{\cal C}(v,s) = \mathcal{L}v(\mathcal{I},
\mathcal{X}) - rv(\mathcal{I},\mathcal{X})
+ F(\mathcal{I},\mathcal{X}), \\
& \mathcal{R}_{\cal S}(v,s) = v(\mathcal{I},\mathcal{X})
-G(\mathcal{I},\mathcal{X}), \\
& \mathcal{R}_{\partial {\cal S}, 1}(v,s) = v
\bigl(s(\mathcal{X}),\mathcal{X}\bigr)-
G\bigl(s(\mathcal{X}),\mathcal{X}\bigr), \\
& \mathcal{R}_{\partial {\cal S}, 2}(v,s) = v_{\pi}
\bigl(s(\mathcal{X}),\mathcal{X}\bigr) +
\frac{\delta(\mathcal{X})' (\partial v/\partial {\bf x})
\bigl(s(\mathcal{X}),\mathcal{X}\bigr)
 }{ s(\mathcal{X}) \bigl(1-s(\mathcal{X})\bigr)
\varrho(\mathcal{X}) } -G_{\pi}\bigl(
s(\mathcal{X}),\mathcal{X}\bigr).
\end{align*}

\noi Equation \eqref{eq:variational_inequality} defines 
the complementarity structure of the FBP. Because 
the continuation and stopping regions are unknown 
a priori, optimizing this structure directly is highly 
challenging for neural networks. Without explicit 
regularization, the networks can easily exploit the 
variational inequality and converge to degenerate 
solutions where almost the entire domain is 
erroneously assigned to just one region, 
minimizing the loss by satisfying only a single 
operational regime. To prevent this behavior and 
ensure the model accurately learns the boundary 
separating the two regions, we introduce the 
following residuals that strictly enforces 
the variational inequality,
\begin{align*}
\widetilde{\mathcal{R}}_{\cal C}(v,s) &=
\max\bigl(-\mathcal{R}_{\cal C}(v,s),0\bigr),\\
\widetilde{\mathcal{R}}_{\cal S}(v,s) &=
\min\bigl(-\mathcal{R}_{\cal S}(v,s),0\bigr), \\
\widetilde{\mathcal{R}}_{{\cal C}, {\cal S}}
(v,s) &= \Phi\bigl(\mathcal{R}_{\cal C}(v,s),
-{\mathcal{R}}_{\cal S}(v,s)\bigr),
\end{align*}

\noi where $\Phi(a,b)=a+b-\sqrt{a^2+b^2}$ is the 
Fischer-Burmeister function. An important property 
of the Fischer-Bermeister function is that $\Phi(a,b)=0$ 
if and only if $a\geq 0$, $b\geq 0$, and $ab=0$.
Then
for any $v(\pi,\mathbf{x})$ $\in$ $C^2([0,1]\times
\mathbb{R}^n,\mathbb{R})$ and $s(\mathbf{x}) \in
C^1(\mathbb{R}^n,[0,1])$, we define the loss functional
$\mathbb{F}: C^2([0,1]\times \mathbb{R}^n,\mathbb{R})
\times C^1(\mathbb{R}^n,[0,1]) \longrightarrow
[0,\infty]$ as follows:
\begin{equation}\label{eq:loss_functional}
\begin{aligned}
 \mathbb{F}(v,s) &= \mathbb{E} \Bigl\{ {\cal R}_{\cal C
}^{2} \mathbf{1}_{\{({\cal I}, {\cal X})\in {\cal C}\}} 
+ \mathcal{R}_{\cal S}^{2} {\bf 1}_{\{({\cal I},
{\cal X}) \in {\cal S}\}} 
+ \big( {\cal R}_{\partial {\cal S}, 1}^{2} +
{\cal R}_{\partial {\cal S}, 2}^{2} \big)
+ \widetilde{\cal R}_{\cal C}^{2} {\bf 1}_{\{(
{\cal I}, {\cal X})\in {\cal C}\}} \\
& \qquad \qquad +\widetilde{\mathcal{R}}_{S}^{2} \mathbf{1}_{
\{(\mathcal{I},\mathcal{X})\in\mathcal{S}\}} +
\widetilde{\cal R}_{{\cal C}, {\cal S} }^{2} \Bigr\}.
\end{aligned}
\end{equation}

\begin{theorem}\label{thm.inf.sto.opt}
Let $n\ge 1$. Suppose that $G, F\in C^2( \mathbb{R}^{n+1}, 
\mathbb{R})$ and let $v \in C^2(\mathbb{R}^{n+1},\mathbb{R})$,
$S\in C^1(\mathbb{R}^n,[0,1])$. 
Assume that there exist strictly positive probability
densities $\widetilde{\pi}\in L^1([0,1])$,
$\widetilde{x}\in L^1(\mathbb{R}^n)$
such that, for every $A\in\mathcal{B}([0,1])$ and
$B\in\mathcal{B}(\mathbb{R}^n)$,
$\mathbb{P}(\mathcal{I}\in A) =
\int_A \widetilde{\pi}(\pi)\,d\pi$,
$\mathbb{P}(\mathcal{X}\in B)=
\int_B \widetilde{x}(x)\,dx$.
Assume that all terms appearing in the residuals are 
well-defined and square integrable.
Then the following two statements are equivalent:
\begin{enumerate}
\item [(i)] $\mathbb{F}(V,S)=0 =\inf_{\substack{
v\in C^2([0,1]\times\mathbb{R}^n,\mathbb{R})\\
s\in C^1(\mathbb{R}^n,[0,1])}}\mathbb{F}(v,s)$.

\item [(ii)] $(V,S)$ is a classical solution of the 
FBP \eqref{eq:fb_pde}--\eqref{eq:fb_smfit}.
\end{enumerate}
\end{theorem}

Under the above formulation, the value function $V(\pi,\mathbf{x})$
and the stopping boundary $S(\mathbf{x})$ solve the 
infinite-dimensional stochastic optimization problem
in Theorem \ref{thm.inf.sto.opt}(i). That is, $\bigl(V(\pi,\mathbf{x}),
S(\mathbf{x})\bigr)$ is the unique global minimizer 
of the functional $(v,s) \longmapsto \mathbb{F}(v,s)
\in [0,\infty]$.
A deep learning-based approximation method replaces 
the function spaces $C^2([0,1]\times\mathbb{R}^n,
\mathbb{R})$ and $C^1(\mathbb{R}^n,[0,1])$
by classes of deep neural networks with a fixed 
and sufficiently large architecture. This converts 
the infinite-dimensional optimization problem
in Theorem \ref{thm.inf.sto.opt}(i) into a finite-dimensional
optimization problem over the parameters of the 
neural networks for $V(\pi,\mathbf{x})$ and
$S(\mathbf{x})$.

\subsection{Neural network approximation}
\label{sec:neural_network_approximation}
We approximate the value function $V(\pi, {\bf x})$ 
and the unknown stopping boundary $S({\bf x})$ 
using two separate deep multilayer feedforward 
neural networks (FNNs),denoted as 
$\widehat{V}(\pi, {\bf x}; \btheta_V)$ and 
$\widehat{S}({\bf x}; \btheta_{S})$ with parameter 
sets $\btheta_V$ and $\btheta_{S}$, respectively.
(see Figure \ref{app:dual_architecture} in Appendix)
Each network approximates its target function 
through a composition of layer-wise operations of the form
$g_K \circ g_{K-1} \circ \cdots \circ g_l \circ \cdots 
\circ g_1$. . Specifically, the output of the $l$-th hidden 
layer is given by $g_l( {\bf y}) = \phi_l ({\bf W}_l {\bf y} + 
{\bf b}_l)$, where ${\bf W}_l$ and ${\bf b}_l$ represent 
the weights matrices and bias vectors, and 
$\phi_l(\cdot)$ denotes a element-wise nonlinear 
activation function. The complete collection of trainable 
parameters for each model is thus defined as $\btheta
:= \{ {\bf W}_l, {\bf b}_l \}_{l=1}^K$. For convenience,
we denote $\bvartheta = \{ \btheta_V, \btheta_S\}$

\medskip
{\bf Sampling strategy}.
To improve training efficiency in high dimensions, we 
employ a residual-based adaptive sampling strategy. 
Inspired by the active sampling principle 
in~\cite{gao2023active}, this approach is tailored 
to the free-boundary structure of the quickest 
detection problem. We first generate candidate 
spatial points using Sobol sequences, sampling 
corresponding posterior probabilities from a 
mixture distribution that combines uniform
samples on $(0,1)$ with samples concentrated near 
the current estimated boundary: $\pi_i=S_{\psi}(x_i)
+\xi_i$, $\xi_i\sim\mathcal{N}(0,\sigma_\pi^2)$.
Each candidate $(\pi_i,x_i)$ is assigned a composite 
residual score $\eta(\pi_i,x_i)$ and resampled 
with probability $P_i=\frac{(\eta(\pi_i,x_i)+\varepsilon)^2}
{\sum_{j=1}^{N_{\mathrm{cand}}}
(\eta(\pi_j,x_j)+\varepsilon)^2}$ to emphasize regions 
with larger approximation errors. 
These points are further refined via local residual-guided 
Metropolis steps in $(\pi,x)$ with acceptance probability
$\alpha= \min\left\{ 1, \frac{\eta(\pi',x')+\varepsilon}
{\eta(\pi,x)+\varepsilon} \right\}$.
Each training batch combines these adaptively refined 
points with uniform Sobol samples, balancing global 
exploration with increased resolution with high resolution 
near the evolving free boundary.

\medskip
{\bf PDE residual on the stopping boundary}.
In addition, we impose an extra constraint on the 
predicted free boundary. Because equation 
\eqref{eq:fb_pde} holds in the continuation region,
${\cal R}_{\cal C}(v, s)$ approaches zero as the 
state approaches the boundary $\partial {\cal D}$ from 
the continuation side. We therefore evaluate
the continuation-region PDE at the predicted boundary
$\bigl(s(\mathcal{X}),\mathcal{X}\bigr)$, defining the 
corresponding loss as
$\widetilde{\cal R}_{\partial {\cal S}, {\cal C} }(v,s)
=\mathcal{R}_{\cal C} \bigl(v, s(\mathcal{X}) \bigr)$.
This loss penalizes boundary predictions inconsistent 
with the PDE by providing a direct gradient signal to 
the approximation $s$ whenever the residual remains nonzero:
$\mathbb{F}_1(v,s) = \mathbb{F}(v,s) + 
\mathbb{E} \big( \widetilde{\cal R}_{\partial 
{\cal S}, {\cal C} }^2 \big)$.

\subsection{Uncertainty Quantification}
\label{sec:uncert_quant}
The PINN approximations, $\widehat{V}(\pi, {\bf x}; \btheta_V)$
and $\widehat{S}({\bf x}; \btheta_{S})$, are subject to 
several sources of uncertainty. 
We now discuss two uncertainty sources and quantify 
them via numerical experiments in Section \ref{sec:experiments}.

{\bf Parameter uncertainty}.
Let $\widehat{\bvartheta}=\{ \widehat{\btheta}_V,  
\widehat{\btheta}_S\}$  be the neural-network 
parameters obtained by minimizing the total loss function. 
For a fixed stochastic model, $\widehat{\bvartheta}$ 
remains subject to uncertainty arising from the finite 
number of collocation points, random initialization, 
stochastic optimization, network architecture, and 
sampling variability. To quantify this uncertainty, we 
use the deep ensemble approach 
of~\citet{lakshminarayanan2017simple}. Let 
$\widehat{\bvartheta}^{(1)},\dots,\widehat{\bvartheta}^{(M)}$ 
denote the parameters obtained from $M$ independent 
training runs with distinct random initializations and/or 
independently generated collocation samples. The 
corresponding approximations form the ensemble 
$\mathcal{D} = \big\{V_{\widehat{\bsvartheta}^{(m)}}, S_{\widehat{\bsvartheta}^{(m)}} \big\}_{m=1}^M$. 
Pointwise empirical uncertainty intervals for 
\(V(\pi,{\bf x})\) and \(S({\bf x})\) are then constructed 
from the \(\alpha/2\) and \(1-\alpha/2\) empirical 
quantiles of $\mathcal{D}$ to yield a \(100(1-\alpha)\%\) confidence band.

{\bf Neural-Network uncertainty}.
We additionally assess neural-network uncertainty 
using Monte Carlo dropout~\citep{gal2016dropout}, 
applying Bernoulli dropout after the hidden-layer 
activations of both networks. During training, masks 
are resampled at each iteration and held fixed across 
collocation points and repeated evaluations.Post-training, 
predictive variability is estimated by holding the 
learned parameters $\widehat{\bvartheta}$ fixed 
and drawing $M$ independent sets of dropout 
masks. Each set is shared across all evaluation 
and batched points, defining a pair of stochastic approximations 
$V_{\widehat{\bsvartheta}}^{(m)}$ and
$S_{\widehat{\bsvartheta}}^{(m)}$, $m=1,\ldots,M$.
Their pointwise sample means, variances, and 
empirical quantiles are computed as in the deep 
ensemble approach, describing variability across 
dropout realizations rather than independent training runs.

\section{NUMERICAL EXPERIMENTS}
\label{sec:experiments}

This section applies the methodology in Sections 2
and 3 to solve the Bayesian sequential change-point
detection problem \eqref{eq:bayesian_risk} for 
multi- or high-dimensional Wiener processes
and Ornstein-Uhlenbeck (OU) processes. For illustrative
purpose, we assume $F(\pi_s, {\bf x}_s) \equiv c>0$
and $G(\pi_s, {\bf x}_s) = 1-\pi_s$, 
so that the optimal stopping problem 
\eqref{eq:general_optimal_stopping} reduces to 
\begin{equation}\label{cpdetect.opt.stop3} 
\inf_{\tau} E_{\pi, {\bf x}} \Big( 
c \int_0^\tau e^{-rs} \pi_s ds + e^{-r\tau} (1-\pi_\tau) \Big). 
\end{equation}

\subsection{Brownian motion with a change in drift}
\label{subsec:BM_drift}

Consider the quickest detection problem for 
a Brownian motion with constant drift and volatility. 
Specifically, we let $\bmu_i({\bf x}) = \bmu_i \in 
\mathbb{R}^n$, $i=0,1$, $\bsigma({\bf x}) =
\bsigma \in \mathbb{R}^{n\times n}$, and 
$\brho ={\bf I}_n$ in the SDE \eqref{eq:cp_sde}.
This means that the process of ${\bf x}_t$ undergoes
a change from $\bmu_0$ to $\bmu_1$ at time $\theta$.
We solve the problem of quickest change 
detection reduces by addressing the FBP 
\eqref{eq:fb_pde}--\eqref{eq:fb_smfit}.
To evaluate the proposed PINN framework, we 
consider a numerical experiment that admits an 
exact analytical solution obtained by reducing the 
multi-dimensional problem to a one-dimensional form.

\begin{example}\label{exam:bm}
Consider the parameter setting with $\lambda=1$, 
$r=0$, $c=1$, $\bmu_0={\bf 0}$, $\bmu_1 =
\frac{1}{\sqrt{n}}\mathbf{1}_n$, $\bsigma={\bf I}_n$, 
and $\brho={\bf I}_n$. This implies that $\varrho
\equiv 1$. We use the proposed PINN to solve the
FBP \eqref{eq:fb_pde}--\eqref{eq:fb_smfit} for
$n\in\{1,5,50,100,500\}$.
The convergence of the training loss and the 
geometry of the estimated stopping boundaries
$\widehat{S}(\mathbf{x})$ are presented in 
Figure \ref{fig:Winner_process_results}. Since the 
state space is high-dimensional (${\bf x}=
(x^{(1)}, x^{(2)}, \dots, x^{(n)}) \in \mathbb{R}^n$), 
we visualize the estimated stopping 
boundary by plotting the surface $\widehat{S}({\bf x})$ 
as a function of the first two coordinates 
$x^{(1)}, x^{(2)} \in (-4,4)$, while holding the remaining 
components $x^{(3)}, \dots, x^{(n)}$ fixed.
\end{example}
\begin{figure}[ht] 
\centering
\includegraphics[scale=0.26]{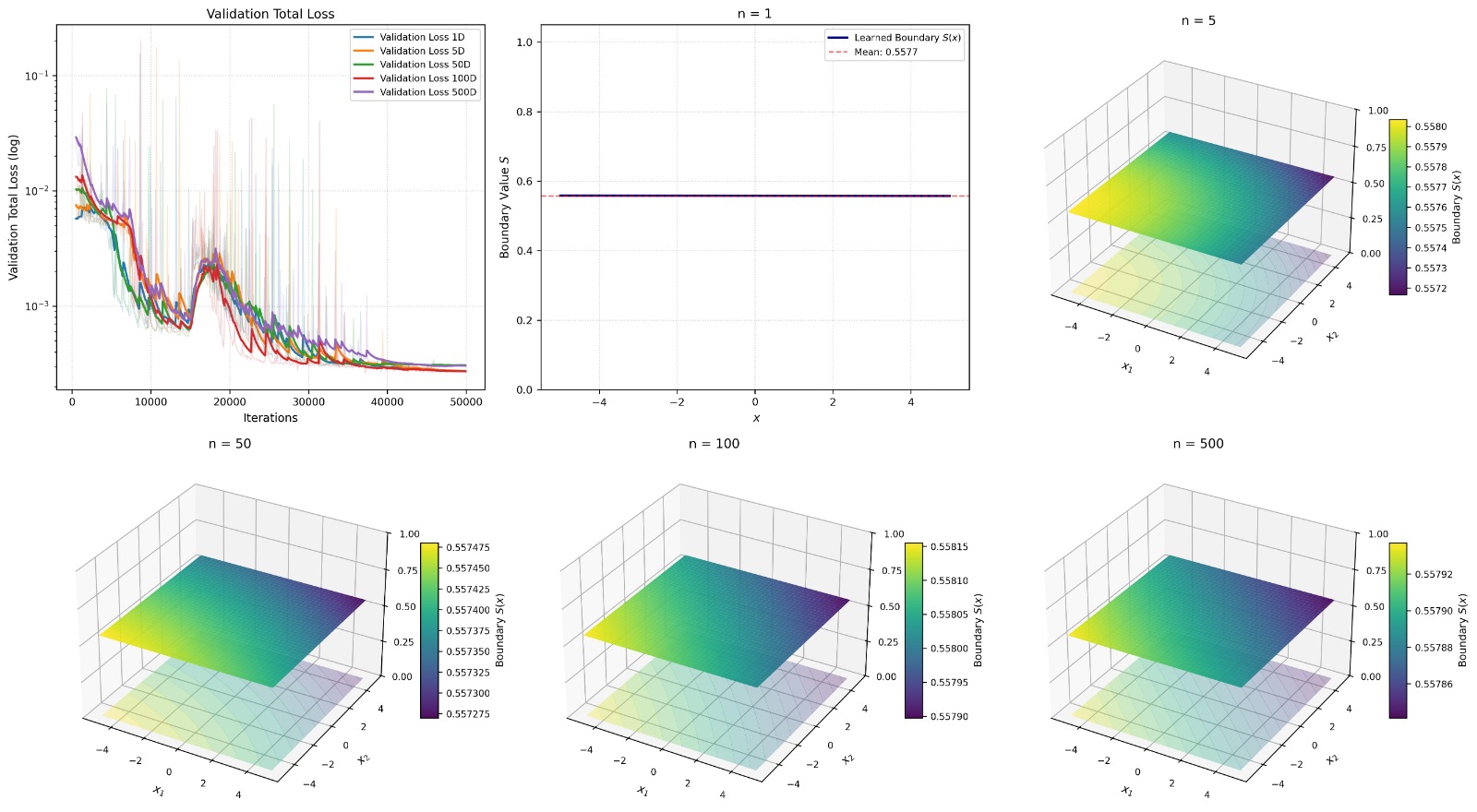}
\caption{Training loss (top-left) and $\widehat{S}(\mathbf{x})$ 
(remaining panels) in Example \ref{exam:bm}.}
 \label{fig:Winner_process_results}
\end{figure}

The optimal stopping boundary $S({\bf x})$ in
Example \ref{exam:bm} has an exact solution. 
Consider the transformation $y_t = \frac{1}{\sqrt{n}}
{\bf 1}_n' {\bf x}_t$, then the process $y_t$ follows
the SDE: $dy_t =dw_t$ for $t\le \theta$ and
$dy_t = dt + dw_t$ for $t > \theta$,
in which $w_t$ is the one-dimensional standard
Brownian motion starting from 0. When $r=0$, the
quickest change detection problem \eqref{cpdetect.opt.stop3} 
reduces to $V(\pi, {\bf x}) = V(\pi) = \inf_{\tau} E_{\pi} \big[
1-\pi_{\tau} + c\int_0^{\tau} \pi_s\,ds \big]$ for all
${\bf x}\in \mathbb{R}^n$. According to \cite{Shiryaev1965},
The continuation region $\mathcal{C}$ and the stopping 
region $\mathcal{S}$ are expressed as
$\mathcal{C} = \left\{ \pi\in(0,S) \right\}$,
$\mathcal{S} = \left\{ \pi\in(S,1) \right\}$,
where $S\in(0,1)$ denotes the unknown stopping boundary.
The associated one-dimensional FBP is given by
$$
\frac{\varrho}{2} \pi^2(1-\pi)^2 V''(\pi) +
\lambda(1-\pi)V'(\pi)+ c_0\pi =0,
$$
$$
V(\pi) < 1-\pi, \mbox{ } \pi\in\mathcal{C}, \qquad
V(\pi) = 1-\pi, \mbox{ } \pi\in\mathcal{S},
$$
$$
V(S) = 1-S, \qquad V'(S) = -1.
$$

\noi \cite{Shiryaev1965} provided an analytical solution 
to the above one-dimensional FBP. Specifically, 
the value function is given by
$V(\pi) = (1-S) + \int_S^{\pi} \psi(\xi) d\xi$
for $\pi\in(0,S)$, and $V(\pi) = 1-\pi$ for $\pi\in(S,1]$.
Here, 
$$
\psi(\pi) = -\frac{2c_0}{\varrho} e^{
-\frac{2\lambda}{\varrho}\alpha(\pi) }
\int_0^{\pi} \frac{ e^{\alpha(\xi)} }{
\xi(1-\xi)^2 } d\xi.
$$

\noi and $\alpha(\pi) = \log\big( \pi/(1-\pi) 
\big) -1/\pi$, 
The stopping boundary $S$ is the unique root of
$\psi(S)=-1$.

Under the parameter setting in Example \ref{exam:bm}, 
the optimal stopping boundary is $S(\mathbf{x})
\equiv 0.5560$ for all ${\bf x} \in \mathbb{R}^n$.
We use this exact solution as a benchmark 
to evaluate the PINN estimate $\widehat{S}({\bf x})$ 
in Example \ref{exam:bm}.
The PINN estimates remain in the range $0.557$--$0.558$,
compared with the exact solution $S(\mathbf{x})=0.5560$, 
the resulting absolute errors are on the order of $10^{-3}$.

\subsection{OU process with a change-point in mean 
reversion speed}\label{subsec:ou_process}

We now evaluate quickest detection where ${\bf x}_t$ 
follows a multivariate OU process with a change 
in mean-reversion speed. Setting $\bmu_i({\bf x}) = 
\bkappa_i(\balpha - {\bf x})$ in the SDE \eqref{eq:cp_sde}, 
the mean-reversion profile transitions from $\bkappa_0$ 
to $\bkappa_1$ at time \(\theta \). The matrices 
$\bkappa_i \in \mathbb{R}^{n\times n}$ capture 
both self-reverting tendencies (diagonal entries) and 
cross-reversion couplings (off-diagonal entries), while 
$\balpha \in \mathbb{R}^n$ is the long-term mean vector.
Dense spatial correlations and cross-channel 
dependencies inherently couple the components 
of a multivariate OU process, meaning that local 
structural changes dynamically diffuse throughout 
the system. This global propagation prevents 
decoupled tracking and makes timely change-point 
detection highly challenging.

To evaluate performance of our PINN framework 
under complex spatial and cross-channel 
dependencies, we analyze two configurations of 
the correlation matrix $\brho$ based on the parameter
$\rho$,

\begin{list}{\textbullet}{%
    \setlength{\topsep}{3pt}%
    \setlength{\partopsep}{0pt}%
    \setlength{\itemsep}{2pt}%
    \setlength{\parsep}{0pt}%
    \setlength{\parskip}{0pt}%
}
\item \textbf{Tridiagonal ($\brho_{\rm tri}$):}
$\rho_{ii}=1$, $\rho_{i,i+1}=\rho_{i+1,i}=\rho$,
and $\rho_{ij}=0$ otherwise.

\item \textbf{Toeplitz ($\brho_{\rm toe}$):}
$\rho_{ij}=\rho^{|i-j|}$ for $1\le i,j\le n$.
\end{list}

\begin{example}\label{exam:ou_kappa_tri}
Consider the parameter setting with 
$(\bkappa_0, \bkappa_1)=({\bf I}_n,2{\bf I}_n)$,
$\balpha={\bf 0}$, $\bsigma={\bf I}_n$,
$\lambda=1$, $r=0.05$, $c_0=1$,
$\brho=\brho_{\rm tri}$, and $\rho=0.1$.
We use the proposed PINN to solve the
FBP \eqref{eq:fb_pde}--\eqref{eq:fb_smfit} for
$n\in\{1,5,50,100,500\}$.
Figure \ref{fig:OU_kappa_tri} illustrates the training 
loss evolution alongside the estimated stopping 
boundaries  $\widehat{S}(\mathbf{x})$.
\end{example}

\begin{example}\label{exam:ou_kappa_toe}
Consider the same parameter setting as 
Example \ref{exam:ou_kappa_tri} except 
we use the Toeplitz correlation matrix
$\brho=\brho_{\rm toe}$ with $\rho=0.1$. 
We apply the proposed PINN framework to solve the FBP
\eqref{eq:fb_pde}--\eqref{eq:fb_smfit} across dimensions
$n\in\{1,5,50,100,500\}$; the resulting training trajectories 
and estimated stopping boundaries 
$\widehat{S}(\mathbf{x})$ are displayed in Figure
\ref{fig:OU_kappa_toe}.
\end{example}

\begin{figure}[ht]
\centering
\includegraphics[scale=0.26]{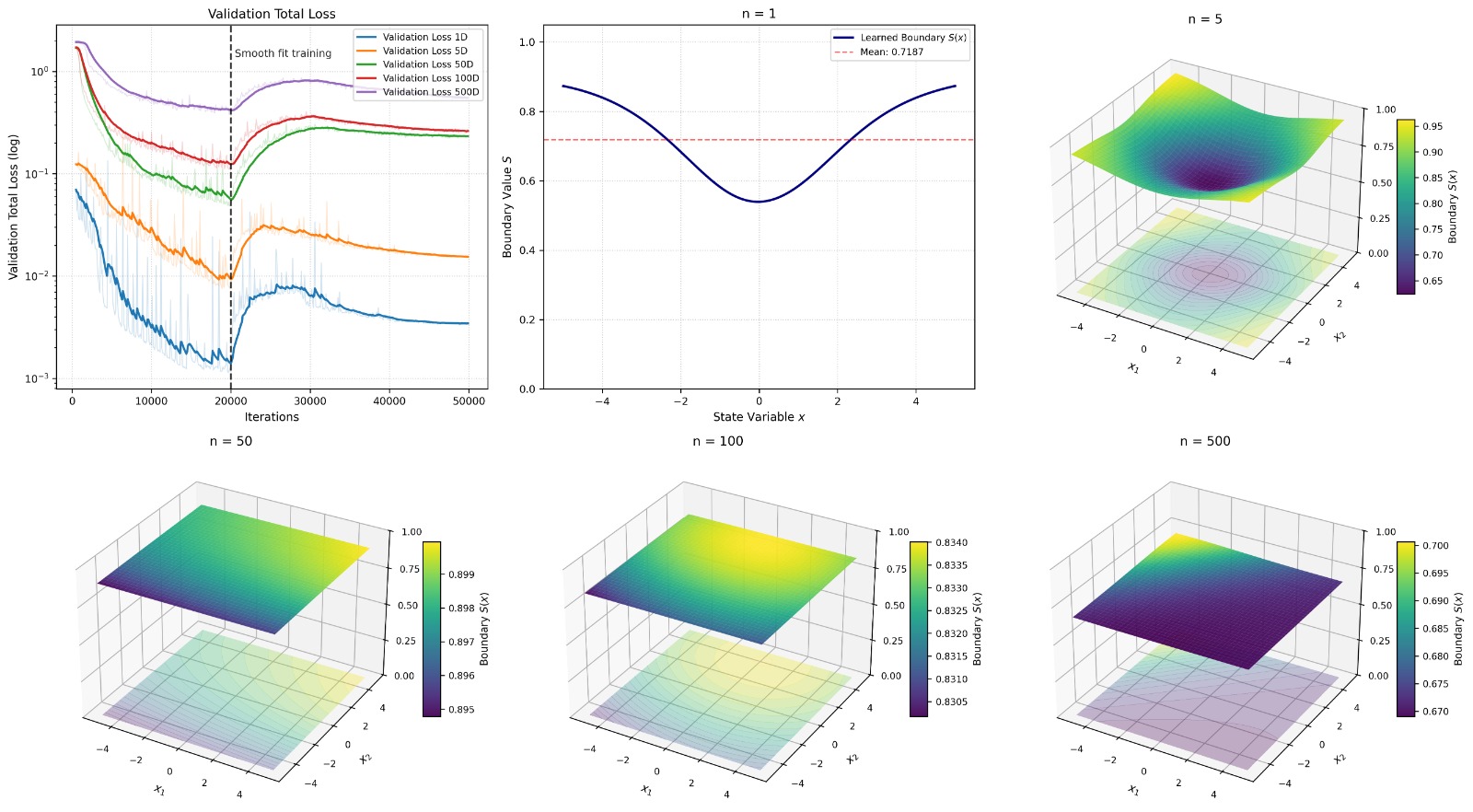}
\caption{Validation total loss (top-left) and $\widehat{S}({\bf x})$ 
(remaining panels) in Example \ref{exam:ou_kappa_tri}.}
\label{fig:OU_kappa_tri}
\end{figure}
\begin{figure}[ht]
\centering
\includegraphics[scale=0.26]{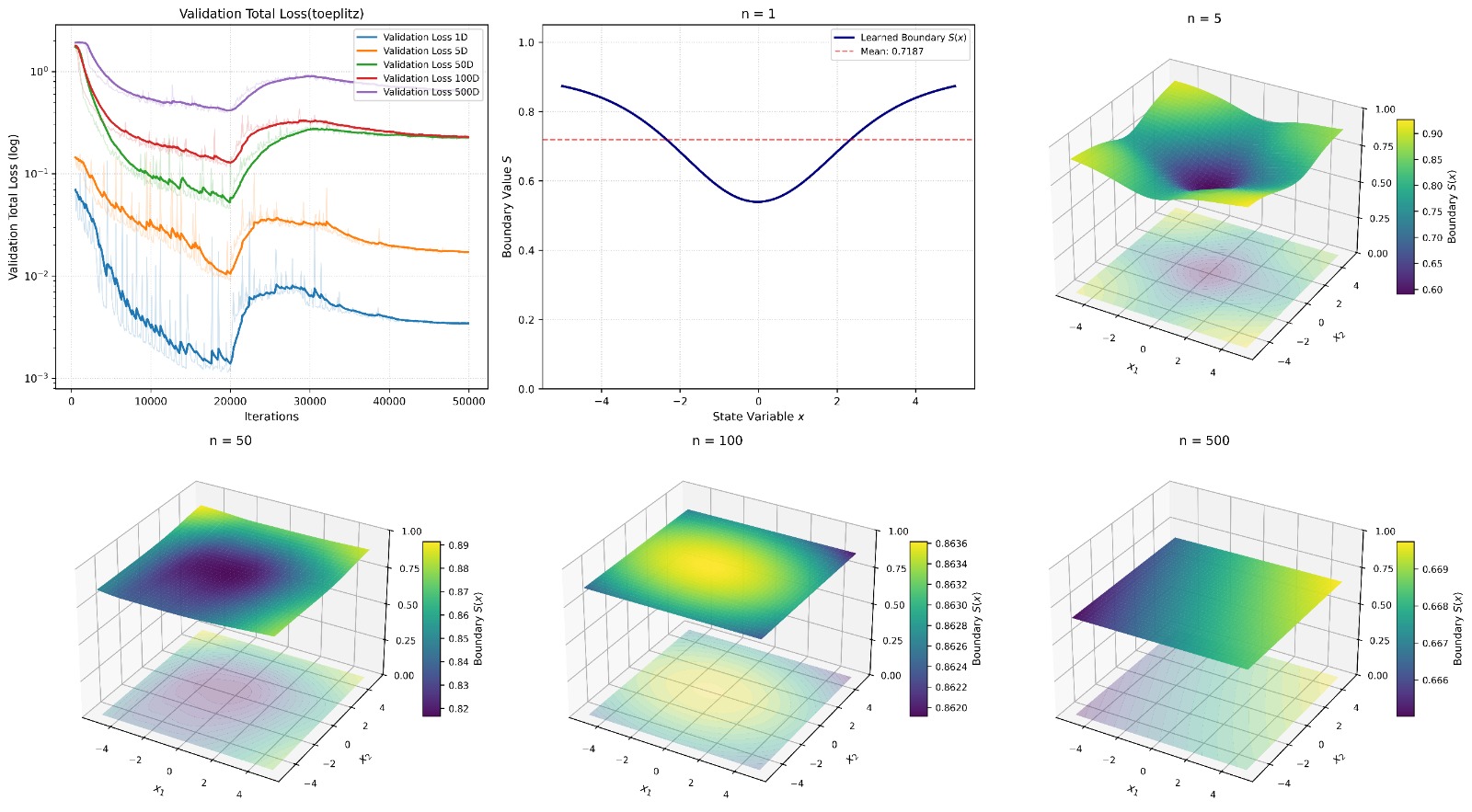}
\caption{Validation total loss (top-left) and $\widehat{S}({\bf x})$ 
(remaining panels) in Example \ref{exam:ou_kappa_toe}.}
\label{fig:OU_kappa_toe}
\end{figure}

Figures~\ref{fig:OU_kappa_tri} and~\ref{fig:OU_kappa_toe} 
illustrate the learned stopping boundaries for 
Examples~\ref{exam:ou_kappa_tri} and~\ref{exam:ou_kappa_toe}, 
respectively. For \(n=1\), both examples yield U-shaped 
boundary estimates, consistent with the symmetry 
of the underlying OU process. For \(n=5\), the learned 
boundary in Example~\ref{exam:ou_kappa_toe} 
exhibits a larger magnitude of variation than that in 
Example~\ref{exam:ou_kappa_tri}. For \(n \ge 100\), 
the displayed \((x_1,x_2)\) slices provide only partial 
views of the boundaries; however, their apparent 
flatness does not imply that the true stopping 
boundaries are state-independent. Because 
analytical benchmarks are unavailable for these 
cases, our assessment focuses on characterizing 
the topological structure of the learned boundaries.

\subsection{OU process with a change-point in 
long-term mean}\label{subsubsec:ou_change_mean}

We next consider the quickest detection problem in 
which the observation $\mathbf{x}_t$ follows 
multivariate OU processes 
with a change-point in the long-term mean. 
Specifically, we let $\bmu_i({\bf x}) = \bkappa
(\balpha_i - {\bf x})$ in the SDE \eqref{eq:cp_sde},
so that the long-term mean of the process undergoes
a change from $\balpha_0$ to $\balpha_1$ at $\theta$. 

\begin{example}\label{exam:ou_alpha}
For the long-term mean change experiment, we set
$\balpha_0=\mathbf{0}$,
$\balpha_1={\bf 1}_n/\sqrt{n}$,
$\bkappa=\bsigma={\bf I}_n$,
$\lambda=1$, $r=0.05$, $c_0=1$, and
$\brho=\brho_{\rm tri}$. The quickest change 
detection problem can be solved by addressing the 
FBP \eqref{eq:fb_pde}--\eqref{eq:fb_smfit}.
In Figure \ref{fig:OU_longterm_results}, the top-left 
panel shows the validation total loss, while the
remaining panels present the estimated stopping 
boundaries $S(\mathbf{x})$
for $n=1,5,50,100,$ and $500$.
\end{example}

\begin{figure}[ht]
\centering
\includegraphics[scale=0.26]{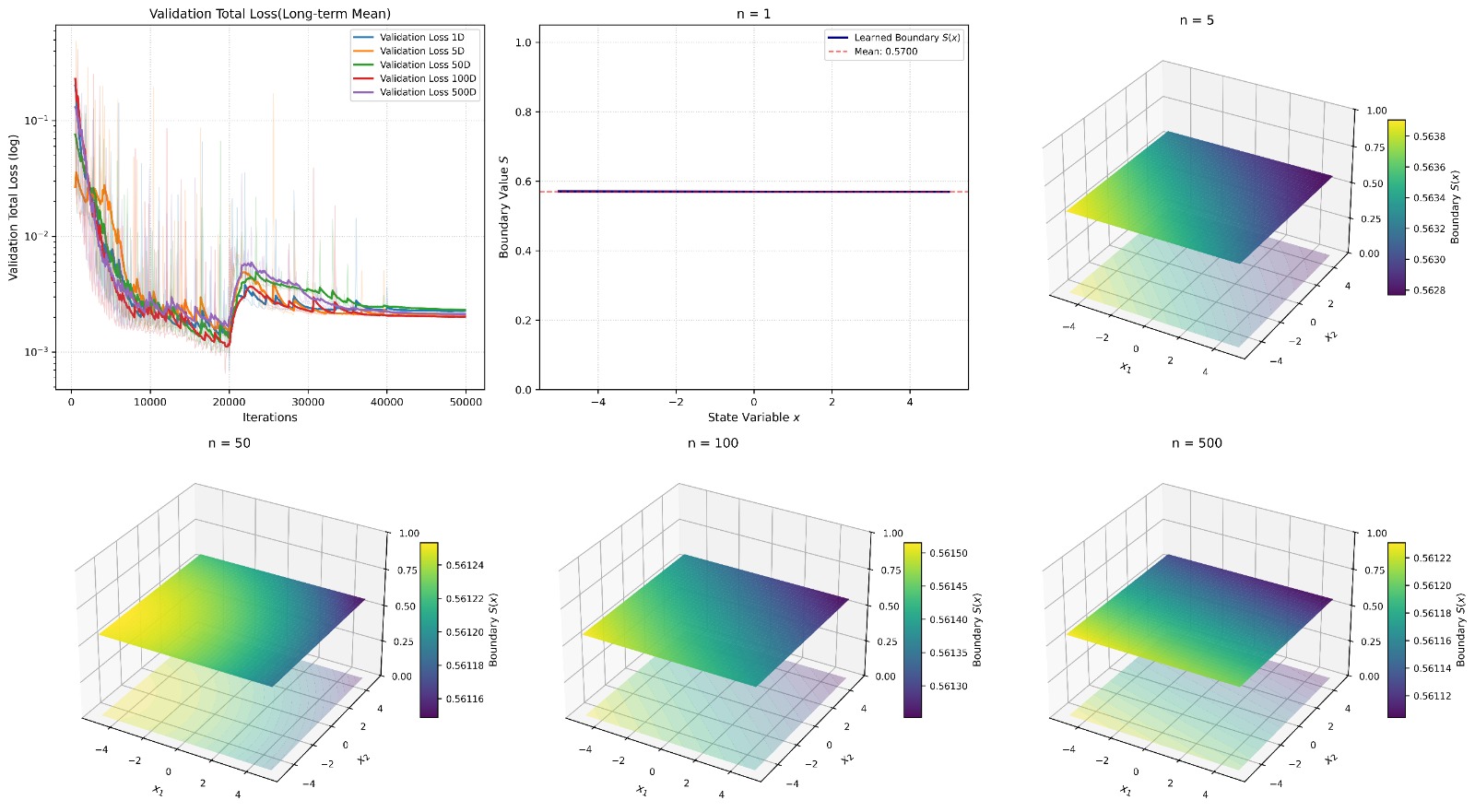}
\caption{Validation total loss (top-left) and $\widehat{S}({\bf x})$ 
(remaining panels) in Example \ref{exam:ou_alpha}.}
\label{fig:OU_longterm_results}
\end{figure}

Figure~\ref{fig:OU_longterm_results} shows the learned 
stopping boundaries for the OU model with a change in 
the long-term mean. As the problem in Example 
\ref{exam:ou_alpha} can be reformulated independently 
of $\mathbf{x}$, it implies a spatially constant 
stopping boundary for each fixed $n$. In particular, 
$\widehat{S}({\bf x})$ for $n=1$ is nearly constant at 
$0.5700$, while the displayed higher-dimensional 
slices exhibit only minor spatial variations around 
$0.563$ for $n=5$ and $0.561$ for $n=50,100,$ and $500$.The 
numerical reference is from $0.5604$ to $0.5690$ for $n=1$
to $n=500$. Comparisons with the PINN estimates are
provided in Appendix~\ref{app:ou_mean_reference}.

\subsection{Uncertainty Quantification and Error Analysis}
\label{sec:UQ}

\begin{figure}[ht]
    \centering
    \includegraphics[scale=0.5]{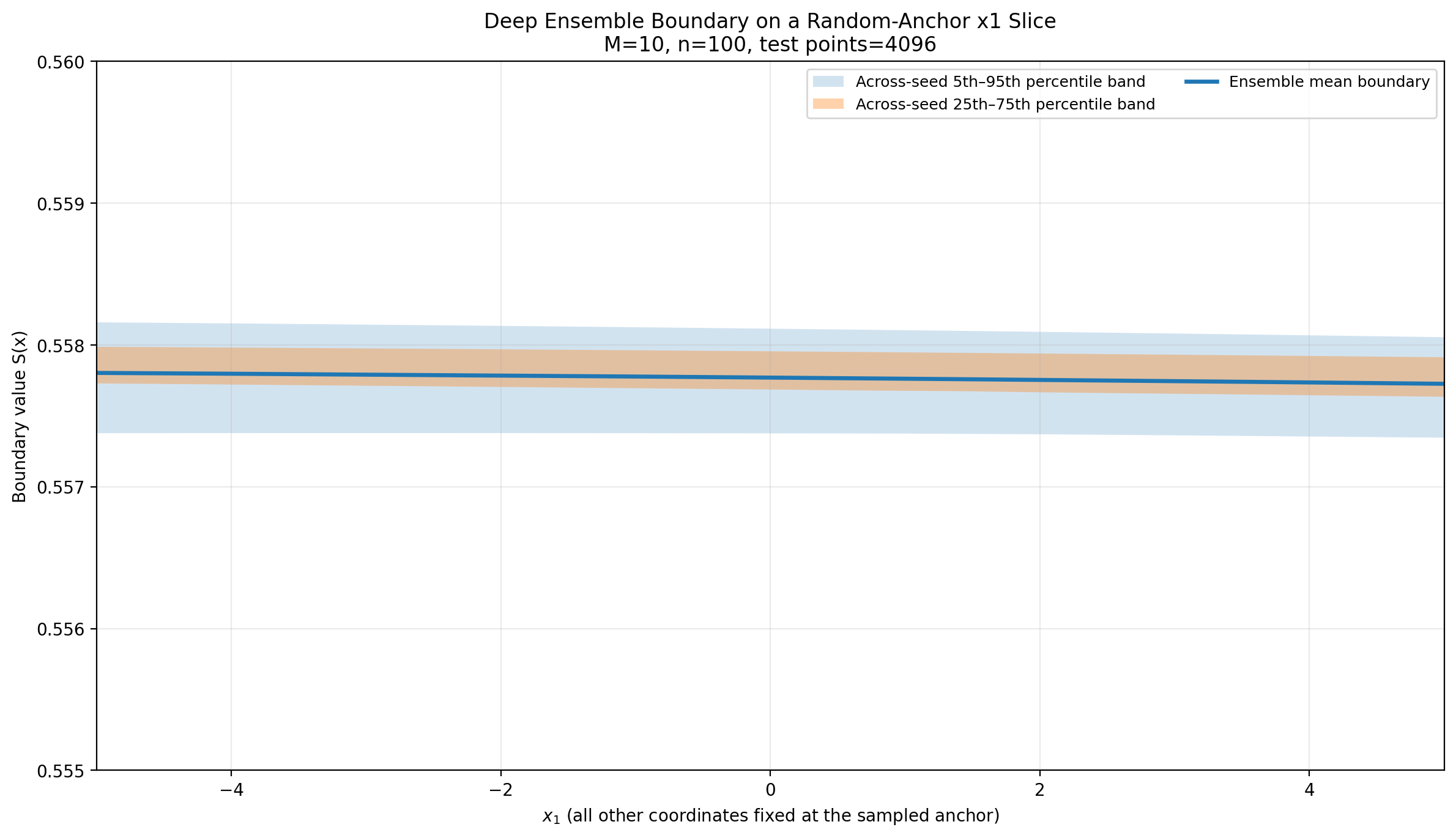}
    \par\vspace{0.5em}
    \includegraphics[scale=0.5]{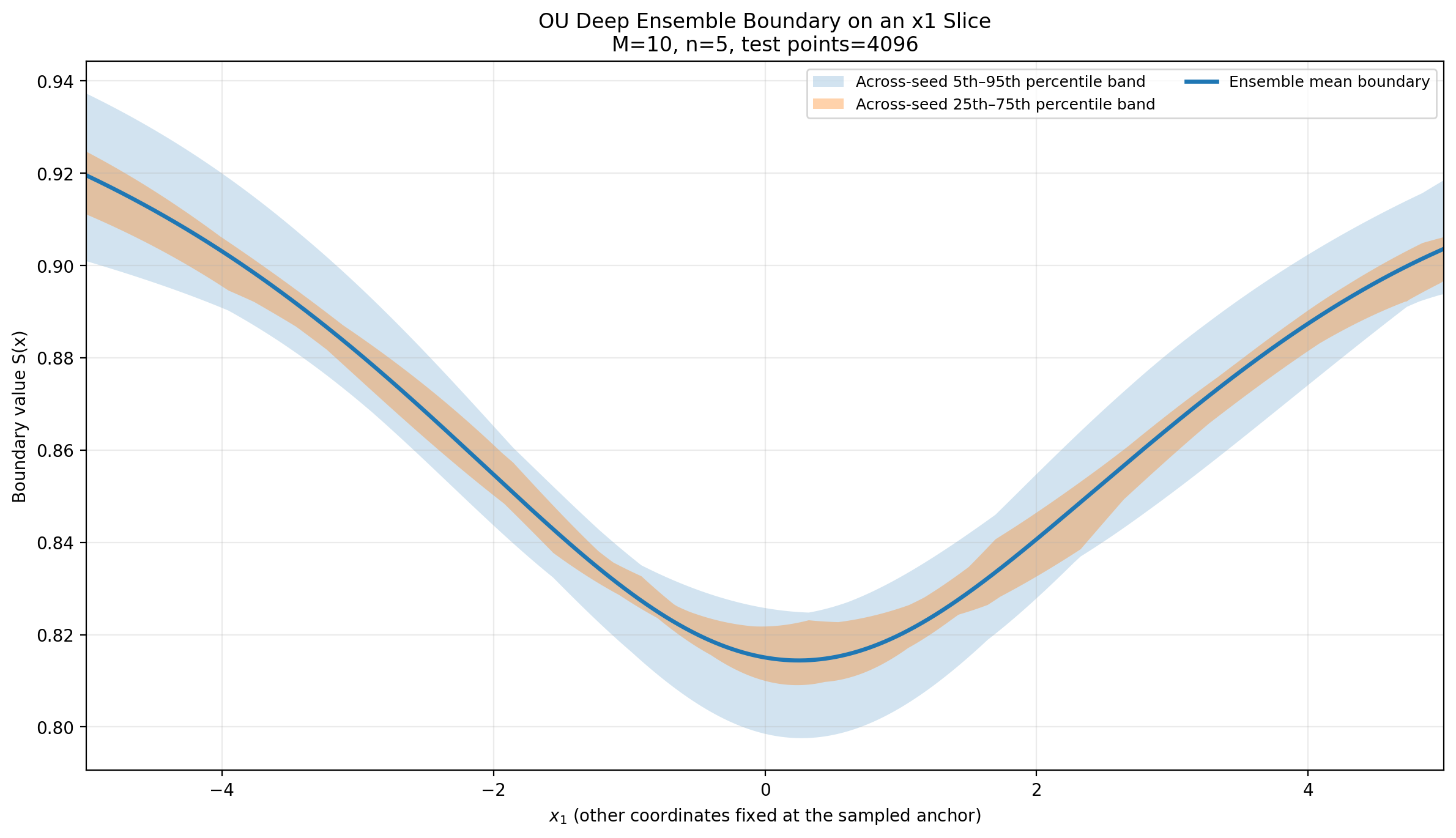}
    \caption{Uncertainty quantification results for the Wiener process
    (top) and the OU process (bottom).}
    \label{fig:combined_uq_results}
\end{figure}

{\bf Deep ensembles}.
We use $M=10$ independent training runs and $4096$ test points
for Example~\ref{exam:bm} ($n=100$) and
Example~\ref{exam:ou_kappa_toe} ($n=5$).
Figure~\ref{fig:combined_uq_results} displays the ensemble mean
boundaries for the Wiener process (top) and the OU process
(bottom) along an $x_1$ slice, with the remaining coordinates
fixed at a sampled anchor in each example.
The shaded regions represent the pointwise 5th--95th and
25th--75th percentile bands across ensemble members.
In the top panel, the ensemble mean remains nearly constant
at approximately $0.5578$, with a 5th--95th percentile band
narrower than $10^{-3}$ throughout the displayed slice.
In the bottom panel, the 5th--95th percentile band has a width
on the order of $10^{-2}$, indicating modest, state-dependent
variability across training runs.
These results indicate close agreement among independently
trained boundary estimates along the displayed slices.

\medskip
{\bf Monte Carlo dropout}.
For Example~\ref{exam:bm} ($n=1$), we perform $100$ stochastic
forward passes at each of $4096$ fixed test points, with a
dropout probability of $0.05$.
Table~\ref{tab:mc_dropout_uq} summarizes the mean, 95th percentile,
and maximum of the pointwise standard deviations for the
stopping boundary and value function.
The boundary estimates exhibit small fluctuations under
dropout, with a maximum pointwise standard deviation of
$7.12 \times 10^{-4}$.

\begin{table}[t]
\centering
\caption{Statistics of pointwise standard 
deviations in Monte Carlo dropout for 
Example \ref{exam:bm} ($n=1$).}
\label{tab:mc_dropout_uq}
\small
\setlength{\tabcolsep}{4pt}
\begin{tabular}{lccc}
\toprule
Quantity & Mean & 95th percentile & Maximum \\
\midrule
$\widehat{S}(\mathbf{x})$
& 0.000710 & 0.000712 & 0.000712 \\
$\widehat{V}(\pi,\mathbf{x})$
& 0.009886 & 0.013369 & 0.013690 \\
\bottomrule
\end{tabular}
\end{table}

\section{CONCLUSION}
\label{sec:conclusion}
This paper develops a PINN framework for Bayesian sequential change-point detection in high-dimensional diffusion processes. 
The Bayesian detection problem is formulated as a continuous-time optimal stopping problem and characterized by a high-dimensional free-boundary problem. 
We then reformulate the free-boundary problem as an infinite-dimensional stochastic optimization problem and approximate the value function and stopping 
boundary jointly using neural networks. The proposed loss function incorporates the governing PDE, boundary related conditions, and 
complementarity constraints, together with high-dimensional sampling strategies.

Numerical experiments for Brownian motion and OU processes demonstrate that the proposed method can accurately approximate benchmark stopping boundaries 
and remains applicable in dimensions up to $1,000$. The experiments also show that the framework can accommodate different changes in diffusion dynamics and 
spatial correlation structures. These results demonstrate the potential of PINNs as a mesh-free approach for high-dimensional free-boundary problems 
arising in continuous-time sequential analysis.

The framework suggests several directions for future research. One is to extend the methodology to other high-dimensional sequential analysis problems 
that admit optimal stopping and free-boundary formulations. Another is to develop more effective adaptive sampling and boundary-focused training strategies for 
high-dimensional state spaces. Finally, systematic uncertainty quantification and approximation-error analysis could provide further information on the reliability of 
the learned value functions and stopping boundaries.

\clearpage
\appendix
\thispagestyle{empty}
\onecolumn

\section{ADDITIONAL DERIVATIONS}


\subsection{Proof of Theorem \ref{thm.inf.sto.opt}}

\textbf{Proof that (ii) implies (i).}
Suppose that $(V,S)$ is a classical solution of the 
free-boundary problem.
In the continuation region $\mathcal{C}_S$, the PDE
\eqref{eq:fb_pde} gives $\mathcal{R}_{\mathcal C}(V,S)
(\pi,x)=0$, $(\pi,x)\in\mathcal{C}_S$.
Consequently, $\mathcal{R}_{\mathcal C}(V,S)^2 \
\mathbf{1}_{\{ (\mathcal{I},\mathcal{X})\in\mathcal{C}_S
\}} =0$ almost surely.
Moreover, the variational inequality gives
$\mathcal{R}_{\mathcal C}(V,S)\geq0$. Therefore,
$\widetilde{\mathcal R}_{\mathcal C}(V,S)=
\max\left\{-\mathcal{R}_{\mathcal C}(V,S),0\right\}
=0$. Similarly, in the stopping region $\mathcal{S}_S$, 
the value matching condition gives
$\mathcal{R}_{\mathcal S}(V,S)(\pi,x)=0$,
$(\pi,x)\in\mathcal{S}_S$. Hence
$\mathcal{R}_{\mathcal S}(V,S)^2\mathbf{1}_{\{
(\mathcal{I},\mathcal{X})\in\mathcal{S}_S \}} =0$
almost surely.
Since $\mathcal{R}_{\mathcal S}(V,S)\leq 0$,
we also have $\widetilde{\mathcal R}_{\mathcal S}(V,S)
= \max\left\{ \mathcal{R}_{\mathcal S}(V,S),0 \right\}
=0$.
The value-matching condition \eqref{eq:fb_value_matching} 
gives $V(S(x),x)=G(S(x),x)$, $x\in\mathbb{R}^n$.
Therefore, $\mathcal{R}_{\partial\mathcal S,1}(V,S)(x)
=0$ for every $x\in\mathbb{R}^n$.
Likewise, the smooth-fit condition
\eqref{eq:fb_smfit} gives
$\mathcal{R}_{\partial\mathcal S,2}(V,S)(x)=0$
for every $x\in\mathbb{R}^n$.
Finally, the variational inequality gives
$\mathcal{R}_{\mathcal C}(V,S)\geq0$,
$-\mathcal{R}_{\mathcal S}(V,S)\geq 0$,
and $\mathcal{R}_{\mathcal C}(V,S)
\left[-\mathcal{R}_{\mathcal S}(V,S) \right] =0$.
The Fischer--Burmeister property
therefore implies $\widetilde{\mathcal R}_{\mathcal C,
\mathcal S}(V,S) = \Phi\left( \mathcal{R}_{\mathcal C}(V,S),
-\mathcal{R}_{\mathcal S}(V,S) \right) =0$.
Thus every nonnegative term in the definition of 
the loss functional vanishes almost surely, and 
consequently $\mathbb{F}(V,S)=0$. Since
$\mathbb{F}(v,s)\geq 0$
for every admissible pair $(v,s)$, it follows that
$0\leq \inf_{v,s}\mathbb{F}(v,s) \leq \mathbb{F}(V,S)=0$.
Hence $\mathbb{F}(V,S) = \inf_{v,s}\mathbb{F}(v,s)
=0$. Therefore, (i) holds.

\medskip

\noindent
\textbf{Proof that (i) implies (ii).}
Suppose now that $\mathbb{F}(V,S)=0$. Since the 
loss functional is a sum of nonnegative terms, every
individual nonnegative contribution has expectation 
equal to zero.

We first consider the continuation-region residual.
Define $H_{\mathcal C}(\pi,x)=
\mathcal{R}_{\mathcal C}(V,S)(\pi,x)^2
\mathbf{1}_{\mathcal C_S}(\pi,x) +
\widetilde{\mathcal R}_{\mathcal C}(V,S)(\pi,x)^2$.
Since $H_{\mathcal C}\geq 0$ and $\mathbb{E}\left[
H_{\mathcal C}(\mathcal I,\mathcal X)
\right]=0$, it implies that
$H_{\mathcal C}(\pi,x)=0$ on the continuation region. 
In particular, $\mathcal{R}_{\mathcal C}(V,S)(\pi,x)=0$,
$(\pi,x)\in\mathcal{C}_S$. Hence the PDE 
\eqref{eq:fb_pde} holds in the continuation region.
Next, define $H_{\mathcal S}(\pi,x)
=\mathcal{R}_{\mathcal S}(V,S)(\pi,x)^2
\mathbf{1}_{\mathcal S_S}(\pi,x) +
\widetilde{\mathcal R}_{\mathcal S}(V,S)(\pi,x)^2$.
Again, $H_{\mathcal S}\geq 0$ and
$\mathbb{E}\left[ H_{\mathcal S}(\mathcal I,\mathcal X)
\right]=0$. It yields
$\mathcal{R}_{\mathcal S}(V,S)(\pi,x)=0$
throughout the stopping region. Hence,
the equation \eqref{eq:fb_stopping} holds.

We next recover the variational inequalities. Since
$\mathbb{E}\big[ \widetilde{\mathcal R}_{\mathcal C}
(V,S)^2 \big]=0$, it gives
$\widetilde{\mathcal R}_{\mathcal C}(V,S)(\pi,x)=0$
for every $(\pi,x)$. Hence $\max\left\{
-\mathcal{R}_{\mathcal C}(V,S),0
\right\}=0$, which implies
$\mathcal{R}_{\mathcal C}(V,S)\geq 0$. Similarly,
$\mathbb{E}\big[
\widetilde{\mathcal R}_{\mathcal S}(V,S)^2
\big]=0$ implies
$\widetilde{\mathcal R}_{\mathcal S}(V,S)(\pi,x)=0$.
Consequently, $\mathcal{R}_{\mathcal S}(V,S)\leq 0$,
and equation \eqref{eq:fb_inequality} holds.
Furthermore, $\mathbb{E}\big[\widetilde{\mathcal 
R}_{\mathcal C,\mathcal S}(V,S)^2\big]=0$ implies
$\Phi\left( \mathcal{R}_{\mathcal C}(V,S),
-\mathcal{R}_{\mathcal S}(V,S) \right)=0$.
By the defining property of the Fischer--Burmeister 
function, $\mathcal{R}_{\mathcal C}(V,S)\geq 0$,
$-\mathcal{R}_{\mathcal S}(V,S)\geq 0$,
and $\mathcal{R}_{\mathcal C}(V,S) 
\mathcal{R}_{\mathcal S}(V,S) =0$. Thus the complete 
variational inequality and complementarity
condition holds.

It remains to verify the two boundary conditions.
Since $\mathbb{E}\left[ \mathcal{R}_{\partial\mathcal S,1}
(V,S)^2 \right]=0$, we have
$\mathcal{R}_{\partial\mathcal S,1}(V,S)(x)=0$
for $\widetilde{x}$-almost every $x\in\mathbb{R}^n$.
Because $V,G\in C^2([0,1]\times\mathbb{R}^n,\mathbb{R})$
and $S\in C^1(\mathbb{R}^n,[0,1])$, the mapping
$x\longmapsto \mathcal{R}_{\partial\mathcal S,1}(V,S)
(x)$ is continuous. Since $\widetilde{x}$ is strictly 
positive everywhere, a continuous function that 
vanishes $\widetilde{x}$-almost everywhere
must vanish everywhere. Consequently,
$\mathcal{R}_{\partial\mathcal S,1}(V,S)(x)=0$,
$x\in\mathbb{R}^n$. Hence equation 
\eqref{eq:fb_value_matching} holds.
Finally, $\mathbb{E}\left[ \mathcal{R}_{\partial\mathcal 
S,2}(V,S)^2 \right]=0$ implies
$\mathcal{R}_{\partial\mathcal S,2}(V,S)(x)=0$
for $\widetilde{x}$-almost every $x$.
Again, the assumed regularity implies that
$x\longmapsto \mathcal{R}_{\partial\mathcal S,2}(V,S)(x)$
is continuous. Hence
$\mathcal{R}_{\partial\mathcal S,2}(V,S)(x)=0$,
$x\in\mathbb{R}^n$. Therefore, the smooth-fit condition 
\eqref{eq:fb_smfit} holds.

Combining the above results, we conclude
that $(V,S)$ solves the FBP
\eqref{eq:fb_pde}--\eqref{eq:fb_smfit}.
Therefore, (ii) holds. $\hfill \Box$

\subsection{Detailed Mathematical Formulations of the Numerical Examples ~\ref{exam:bm}--\ref{exam:ou_alpha}}
Due to space limitations in the main text, we provide here the
complete mathematical formulations for the numerical
examples considered in Section~\ref{sec:experiments}.
\subsubsection{Example~ \ref{exam:bm}}

Given the SDEs for the posterior probability and the observation
process under measure $\mathbb{P}_{\pi,\mathbf{x}}$, the corresponding
infinitesimal operator in Section~ \ref{subsec:opt_stop} is
\begin{align}
\mathcal{L}_{\mathrm{bm}}
={}&
\frac{\varrho}{2}\pi^2(1-\pi)^2
\frac{\partial^2}{\partial \pi^2}
+
\pi(1-\pi)\boldsymbol{\delta}^{\prime}
\frac{\partial^2}{\partial \pi \partial \mathbf{x}}
+
\lambda(1-\pi)\frac{\partial}{\partial \pi}
\nonumber\\
&+
\frac{1}{2}
\operatorname{tr}
\left(
\boldsymbol{\Sigma}
\frac{\partial^2}
{\partial \mathbf{x}\partial \mathbf{x}^{\prime}}
\right)
+
\left(
\boldsymbol{\mu}_0+\pi\boldsymbol{\delta}
\right)^{\prime}
\frac{\partial}{\partial \mathbf{x}},
\label{eq:bm_generator_appendix}
\end{align}
where
\[
\varrho
=
\boldsymbol{\delta}^{\prime}
\boldsymbol{\Sigma}^{-1}
\boldsymbol{\delta}.
\]

Then the free-boundary problem(See Equation~\ref{eq:fb_system}) for the sequential change-point
detection problem is expressed as
\begin{align}
\mathcal{L}_{\mathrm{bm}}V(\pi,\mathbf{x})
-rV(\pi,\mathbf{x})+c_0\pi
&=0,
&&
\text{for } \pi\in(0,S(\mathbf{x})),
\label{eq:bm_fbp_a}
\\
\left.
V(\pi,\mathbf{x})
\right|_{\pi=S(\mathbf{x})^-}
&=
1-S(\mathbf{x}),
\label{eq:bm_fbp_b}
\\
V(\pi,\mathbf{x})
&=
1-\pi,
&&
\text{for } \pi\in(S(\mathbf{x}),1),
\label{eq:bm_fbp_c}
\\
V(\pi,\mathbf{x})
&<
1-\pi,
&&
\text{for } \pi\in(0,S(\mathbf{x})).
\label{eq:bm_fbp_d}
\end{align}

The corresponding smooth-fit condition is
\begin{equation}
\left.
\left[
\frac{\partial V}{\partial \pi}
+
\frac{1}
{\pi(1-\pi)\varrho}
\boldsymbol{\delta}^{\prime}
\frac{\partial V}{\partial \mathbf{x}}
\right]
\right|_{\pi=S(\mathbf{x})^-}
=
-1.
\label{eq:bm_smooth_fit_appendix}
\end{equation}

\subsubsection{Example~ \ref{exam:ou_kappa_tri} and Example~ \ref{exam:ou_kappa_toe}}
The infinitesimal operator in Section~\ref{subsec:opt_stop} associated with the posterior probability
and the observation process is
\begin{align}
\mathcal{L}_{\mathrm{ou},1}
={}&
\frac{\varrho_{\mathrm{ou},1}(\mathbf{x})}{2}
\pi^2(1-\pi)^2
\frac{\partial^2}{\partial \pi^2}
+
\pi(1-\pi)
(\boldsymbol{\alpha}-\mathbf{x})^{\prime}
(\boldsymbol{\kappa}_1-\boldsymbol{\kappa}_0)^{\prime}
\frac{\partial^2}{\partial \pi \partial \mathbf{x}}
\nonumber\\
&+
\lambda(1-\pi)
\frac{\partial}{\partial \pi}
+
\frac{1}{2}
\operatorname{tr}
\left(
\boldsymbol{\Sigma}
\frac{\partial^2}
{\partial \mathbf{x}\partial \mathbf{x}^{\prime}}
\right)
\nonumber\\
&+
(\boldsymbol{\alpha}-\mathbf{x})^{\prime}
\left(
\boldsymbol{\kappa}_0
+
\pi(\boldsymbol{\kappa}_1-\boldsymbol{\kappa}_0)
\right)^{\prime}
\frac{\partial}{\partial \mathbf{x}}.
\label{eq:ou_kappa_generator_appendix}
\end{align}

For the sequential change-point detection problem, we solve for
the value function $V(\pi,\mathbf{x})$ and the unknown stopping
boundary $S(\mathbf{x})$ from the following free-boundary problem:
\begin{align}
\mathcal{L}_{\mathrm{ou},1}V(\pi,\mathbf{x})
-rV(\pi,\mathbf{x})+c_0\pi
&=0,
&&
\text{for } \pi\in(0,S(\mathbf{x})),
\label{eq:ou_kappa_fbp_a}
\\
\left.
V(\pi,\mathbf{x})
\right|_{\pi=S(\mathbf{x})^-}
&=
1-S(\mathbf{x}),
\label{eq:ou_kappa_fbp_b}
\\
V(\pi,\mathbf{x})
&=
1-\pi,
&&
\text{for } \pi\in(S(\mathbf{x}),1),
\label{eq:ou_kappa_fbp_c}
\\
V(\pi,\mathbf{x})
&<
1-\pi,
&&
\text{for } \pi\in(0,S(\mathbf{x})).
\label{eq:ou_kappa_fbp_d}
\end{align}

The corresponding smooth-fit condition is
\begin{equation}
\left.
\left[
\frac{\partial V}{\partial \pi}(\pi,\mathbf{x})
+
\frac{1}
{\pi(1-\pi)\varrho_{\mathrm{ou},1}(\mathbf{x})}
(\boldsymbol{\alpha}-\mathbf{x})^{\prime}
(\boldsymbol{\kappa}_1-\boldsymbol{\kappa}_0)^{\prime}
\frac{\partial V}{\partial \mathbf{x}}(\pi,\mathbf{x})
\right]
\right|_{\pi=S(\mathbf{x})^-}
=
-1.
\label{eq:ou_kappa_smooth_fit_appendix}
\end{equation}

\subsubsection{Example~ \ref{exam:ou_alpha}}

The infinitesimal operator in Section~\ref{subsec:opt_stop} associated with the posterior probability
and the observation process is
\begin{align}
\mathcal{L}_{\mathrm{ou},2}
={}&
\frac{\varrho_{\mathrm{ou},2}}{2}
\pi^2(1-\pi)^2
\frac{\partial^2}{\partial \pi^2}
+
\pi(1-\pi)
(\boldsymbol{\alpha}_1-\boldsymbol{\alpha}_0)^{\prime}
\boldsymbol{\kappa}^{\prime}
\frac{\partial^2}{\partial \pi \partial \mathbf{x}}
+
\lambda(1-\pi)
\frac{\partial}{\partial \pi}
\nonumber\\
&+
\frac{1}{2}
\operatorname{tr}
\left(
\boldsymbol{\Sigma}
\frac{\partial^2}
{\partial \mathbf{x}\partial \mathbf{x}^{\prime}}
\right)
+
\left(
\boldsymbol{\alpha}_0
+
\pi(\boldsymbol{\alpha}_1-\boldsymbol{\alpha}_0)
\right)^{\prime}
\boldsymbol{\kappa}^{\prime}
\frac{\partial}{\partial \mathbf{x}}.
\label{eq:ou_alpha_generator_appendix}
\end{align}

Then, for the sequential change-point detection problem, we solve for
the value function $V(\pi,\mathbf{x})$ and the unknown stopping
boundary $S(\mathbf{x})$ from the following free-boundary problem:
\begin{align}
\mathcal{L}_{\mathrm{ou},2}V(\pi,\mathbf{x})
-rV(\pi,\mathbf{x})+c_0\pi
&=0,
&&
\text{for } \pi\in(0,S(\mathbf{x})),
\label{eq:ou_alpha_fbp_a}
\\
\left.
V(\pi,\mathbf{x})
\right|_{\pi=S(\mathbf{x})^-}
&=
1-S(\mathbf{x}),
\label{eq:ou_alpha_fbp_b}
\\
V(\pi,\mathbf{x})
&=
1-\pi,
&&
\text{for } \pi\in(S(\mathbf{x}),1),
\label{eq:ou_alpha_fbp_c}
\\
V(\pi,\mathbf{x})
&<
1-\pi,
&&
\text{for } \pi\in(0,S(\mathbf{x})).
\label{eq:ou_alpha_fbp_d}
\end{align}

The corresponding smooth-fit condition is
\begin{equation}
\left.
\left[
\frac{\partial V}{\partial \pi}(\pi,\mathbf{x})
+
\frac{1}
{\pi(1-\pi)\varrho_{\mathrm{ou},2}}
(\boldsymbol{\alpha}_1-\boldsymbol{\alpha}_0)^{\prime}
\boldsymbol{\kappa}^{\prime}
\frac{\partial V}{\partial \mathbf{x}}(\pi,\mathbf{x})
\right]
\right|_{\pi=S(\mathbf{x})^-}
=
-1.
\label{eq:ou_alpha_smooth_fit_appendix}
\end{equation}

\newpage
\section{Additional Numerical Experiments}

\subsection{Figure of Dual Architecture Neural Network}

\begin{figure}[h] 
\centering
\includegraphics[scale=0.6]{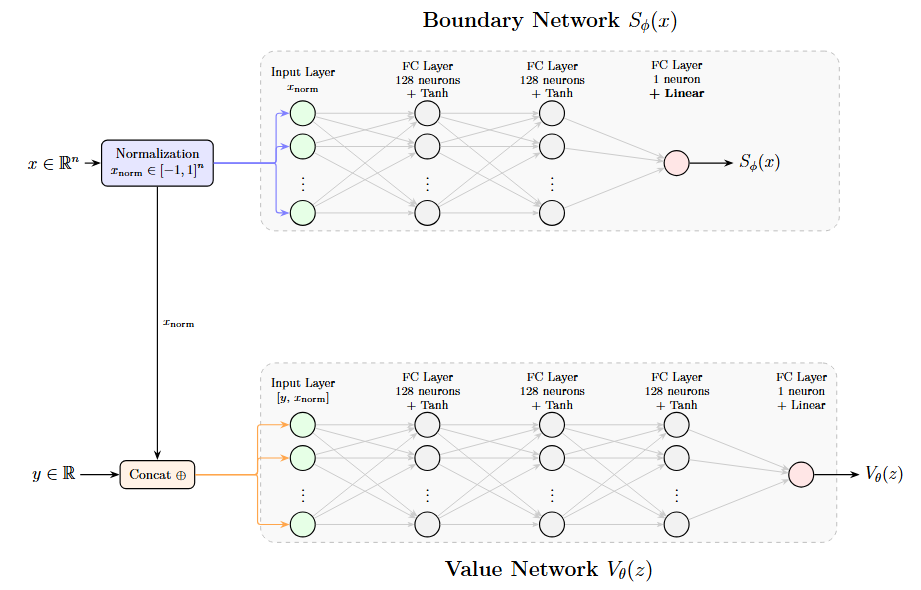}
\caption{Dual architecture neural network used to approximate the value function and stopping boundary.See Section \ref{sec:neural_network_approximation} for details.}
 \label{app:dual_architecture}
\end{figure}

\newpage

\subsection{Figure of More Uncertainty Quantification Results}

We also provide additional uncertainty quantification results for the Wiener process and the OU process in Section \ref{sec:UQ}.

\begin{figure*}[t]
    \centering

    \begin{minipage}[t]{0.48\textwidth}
        \centering
        \includegraphics[width=0.8\linewidth]
        {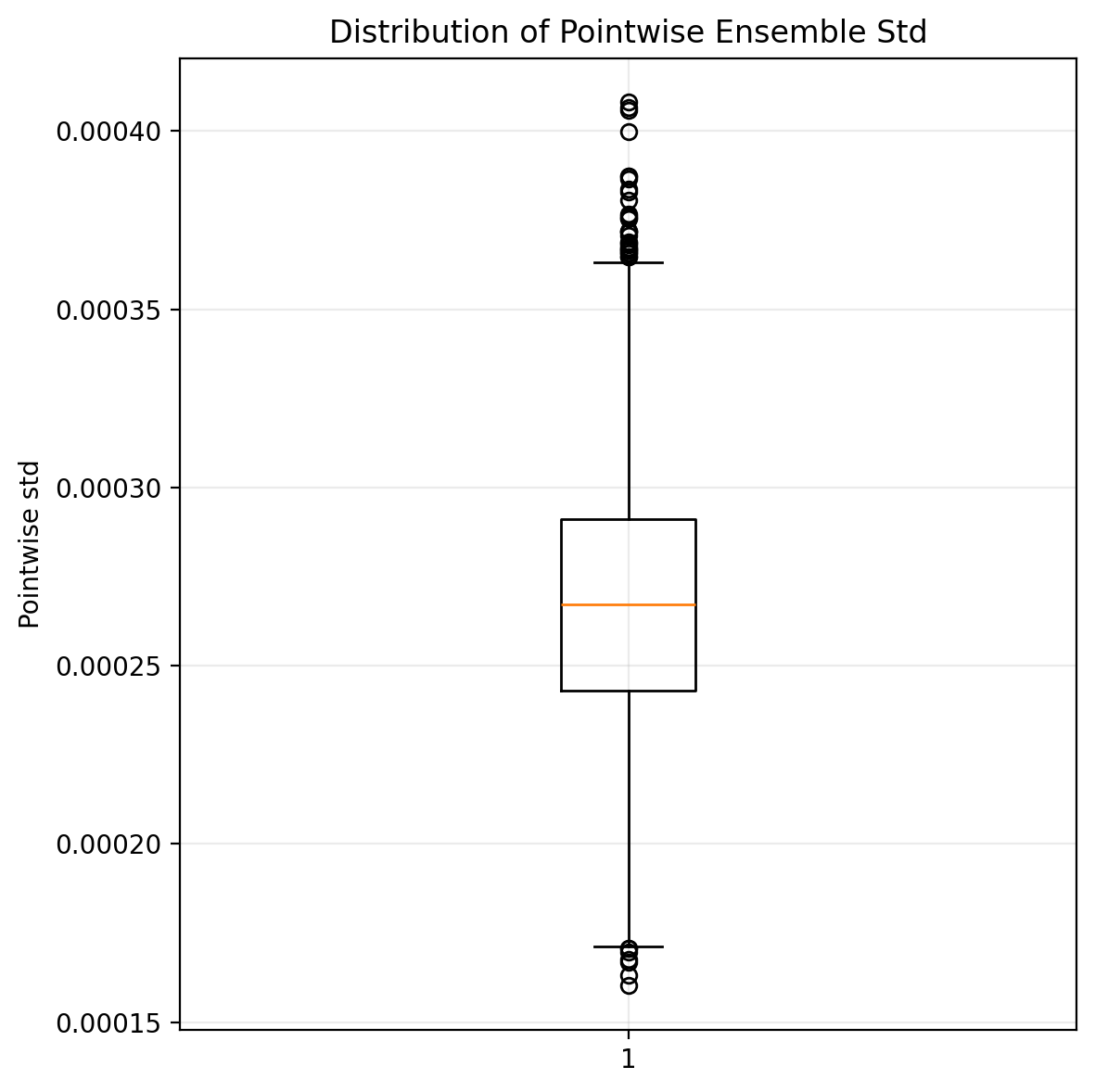}
        \caption{Box plot of the uncertainty quantification results
        for the Wiener process in Example~\ref{exam:bm}.}
        \label{app:WP_UQ_boxplot}
    \end{minipage}
    \hfill
    \begin{minipage}[t]{0.48\textwidth}
        \centering
        \includegraphics[width=0.8\linewidth]
        {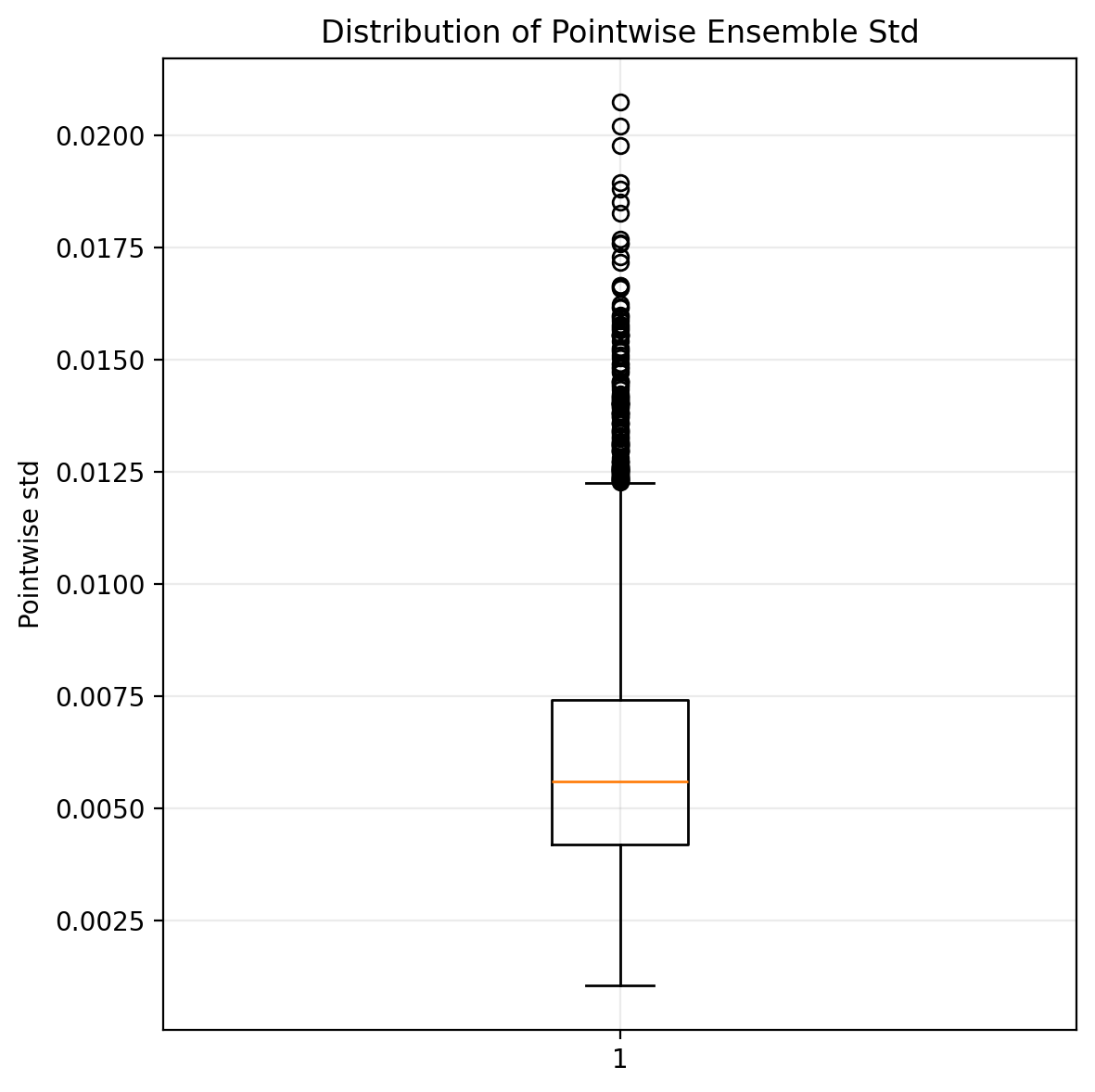}
        \caption{Box plot of the uncertainty quantification results
        for the OU process in Example~\ref{exam:ou_kappa_toe}.}
        \label{app:OU_UQ_boxplot}
    \end{minipage}

\end{figure*}

Figure~\ref{app:WP_UQ_boxplot} presents the distribution of pointwise
ensemble standard deviations of the learned stopping boundary for
the Brownian-motion setting in Example~\ref{exam:bm}. At each evaluation state,
the standard deviation is computed across independently trained
ensemble members, and the box plot summarizes these values over
the evaluation set. The median is approximately
$2.7\times10^{-4}$, with the central $50\%$ of the values lying
between approximately $2.4\times10^{-4}$ and $2.9\times10^{-4}$.
Several upper-tail outliers with the largest values
reaching approximately $4.1\times10^{-4}$. Overall, the small
magnitude of standard deviations indicates consistent
boundary predictions across independent training runs.

Figure~\ref{app:OU_UQ_boxplot} summarizes the pointwise ensemble
standard deviations of the learned stopping boundary for
Example~\ref{exam:ou_kappa_toe}. The median is approximately $5.6\times10^{-3}$,
with the central $50\%$ of values ranging from
$4.2\times10^{-3}$ to $7.4\times10^{-3}$.
The maximum standard deviation is approximately
$2.1\times10^{-2}$.

\begin{table}[htbp]
\caption{Seed-wise summary statistics of absolute boundary deviations
for the OU setting in Example~\ref{exam:ou_kappa_toe}.
$Q_p$ denotes the empirical $p$-quantile across evaluation states.}
\label{tab:ou_uq_abs_deviation}
\centering
\begin{tabular}{@{}rrrrrrrr@{}}
\toprule
Seed & Mean & Median
& $Q_{0.05}$ & $Q_{0.25}$ & $Q_{0.75}$ & $Q_{0.95}$ & Max \\
\midrule
123456
& 0.006884 & 0.005040 & 0.000440 & 0.002290
& 0.009513 & 0.019150 & 0.050168 \\

1234
& 0.005831 & 0.004549 & 0.000377 & 0.002107
& 0.008100 & 0.015849 & 0.030691 \\

12345678
& 0.005821 & 0.004504 & 0.000393 & 0.001954
& 0.008261 & 0.015811 & 0.041734 \\

2023
& 0.005895 & 0.004819 & 0.000448 & 0.002189
& 0.008192 & 0.015134 & 0.050772 \\

2027
& 0.005684 & 0.004560 & 0.000406 & 0.002093
& 0.007916 & 0.014973 & 0.041095 \\

12345
& 0.005566 & 0.004404 & 0.000421 & 0.002094
& 0.007770 & 0.014699 & 0.034935 \\

1234567
& 0.005144 & 0.004039 & 0.000348 & 0.001784
& 0.007237 & 0.013864 & 0.039323 \\

42
& 0.005356 & 0.004321 & 0.000386 & 0.001986
& 0.007525 & 0.013672 & 0.045785 \\

2004
& 0.004713 & 0.003534 & 0.000326 & 0.001647
& 0.006529 & 0.013167 & 0.031337 \\

1520
& 0.004436 & 0.003516 & 0.000295 & 0.001603
& 0.006192 & 0.011807 & 0.029155 \\
\bottomrule
\end{tabular}
\end{table}

Table~\ref{tab:ou_uq_abs_deviation} reports the seed-wise
summary statistics of absolute boundary deviations
for Example~\ref{exam:ou_kappa_toe}.

For the seed-wise analysis, we compute the absolute deviation of
each model's boundary prediction from the mean prediction of the
remaining ensemble members at each evaluation state:
\[
d_m(x_j)
=
\left|
S_m(x_j)
-
\frac{1}{M-1}\sum_{\ell\neq m}S_\ell(x_j)
\right|,
\]
where $M$ is the ensemble size and $x_j$ denotes an evaluation state.
For each seed, these deviations are summarized across evaluation
states by their mean, median, selected quantiles, and maximum,
as reported in Table~\ref{tab:ou_uq_abs_deviation}.The results show consistent boundary predictions across independent
training runs, with mean absolute deviations ranging from approximately
$4.4\times10^{-3}$ to $6.9\times10^{-3}$ across seeds.
For every seed, the 95th percentile remains below $2.0\times10^{-2}$,
supporting the empirical stability of the learned stopping boundary
over the evaluated states.

\subsection{Numerical benchmark with Example~\ref{exam:ou_alpha}}
\label{app:ou_mean_reference}

For an additional benchmark comparison, we consider the parameter
setting of Example~\ref{exam:ou_alpha} with the tridiagonal correlation matrix. Specifically, we use
$\alpha_0=0$, $\alpha_1=\mathbf{1}_n/\sqrt{n}$,
$\kappa=\sigma=I_n$, $\lambda=c_0=1$, and $r=0.05$.
The correlation matrix $R_n$ is defined by
\[
(R_n)_{ij}=
\begin{cases}
1, & i=j,\\
\rho, & |i-j|=1,\\
0, & \text{otherwise},
\end{cases}
\qquad \rho=0.1.
\]
Both the PINN and the numerical reference calculations use this
same tridiagonal configuration.

\paragraph{Reduction to a scalar observation process.}
The drift difference is the constant vector
$\delta_n=\mathbf{1}_n/\sqrt{n}$, and the observation process satisfies
\[
dX_t=
\left(-X_t+\delta_n\mathbf{1}_{\{\theta\leq t\}}\right)dt
+R_n^{1/2}\,dW_t.
\]
Define
\[
q_n=\delta_n^\top R_n^{-1}\delta_n,
\qquad
Y_t=\frac{\delta_n^\top R_n^{-1}X_t}{\sqrt{q_n}}.
\]
A direct calculation yields
\[
dY_t=
\left(-Y_t+\sqrt{q_n}\mathbf{1}_{\{\theta\leq t\}}\right)dt
+dB_t,
\]
where $B_t$ is a standard one-dimensional Brownian motion.
After whitening and rotating the original observations, only this
coordinate has a drift that changes at $\theta$. Conditional on
the observed initial state, the remaining coordinates are independent
OU processes with dynamics that do not depend on $\theta$.
Thus, $Y_t$ retains the change-point information contained in the
full observation process.

The scalar process changes its long-term mean while retaining its
mean-reversion rate. This constant-signal setting admits a
one-dimensional optimal stopping formulation, as discussed by
\citet[Section~3, Case~0]{GloverPeskir2024}.
Here the posterior probability satisfies
\[
d\pi_t=\lambda(1-\pi_t)\,dt
+\sqrt{q_n}\,\pi_t(1-\pi_t)\,d\widetilde{B}_t,
\]
where $\widetilde{B}_t$ is the innovation Brownian motion.
Since the running and stopping costs depend only on $\pi$,
\[
V_n(\pi,x)=u_n(\pi),
\qquad S_n^*(x)\equiv S_n^*.
\]
The boundary is therefore constant in $x$ for each fixed $n$,
although its value depends on $n$ and $R_n$ through $q_n$.

\paragraph{Initial-value reference calculation.}
Using this reduction and retaining the discount rate $r=0.05$
from our problem, we solve
\begin{equation}
\frac{1}{2}q_n\pi^2(1-\pi)^2u_n''(\pi)
+\lambda(1-\pi)u_n'(\pi)-r u_n(\pi)+c_0\pi=0,
\qquad 0<\pi<S_n^*,
\label{eq:ou_mean_reference_ode}
\end{equation}
subject to
\[
u_n(S_n^*)=1-S_n^*,
\qquad u_n'(S_n^*)=-1,
\qquad \lambda u_n'(0)-r u_n(0)=0.
\]
In the stopping region, $u_n(\pi)=1-\pi$.
Our reference calculation thus solves the discounted problem,
rather than directly importing an undiscounted threshold from
the cited literature.

We first use an initial-value problem (IVP) formulation.
The regular continuation solution is represented as a linear
combination of a homogeneous solution and a particular solution.
These solutions are initialized using tenth-order local expansions
at $\pi=10^{-4}$ and integrated using a Radau IIA method
\citep[Section~IV.8]{HairerWanner1996}, with relative and absolute
tolerances of $10^{-12}$ and $10^{-14}$, respectively.
Value matching determines the linear coefficient, while smooth fit
yields a scalar root condition for the stopping threshold.
The reference calculation does not use PINN predictions.
Table~\ref{tab:ou_mean_ivp} reports the resulting thresholds,
denoted by $S_{\mathrm{IVP}}$.
Each row is obtained by solving a one-dimensional problem with
the effective signal strength associated with the original
observation dimension $n$.

\begin{table}[htbp]
\caption{Numerical reference thresholds for the tridiagonal
OU long-term mean change configuration, obtained using the
initial-value formulation.}
\label{tab:ou_mean_ivp}
\centering
\begin{tabular}{@{}rr@{}}
\toprule
$n$ & $S_{\mathrm{IVP}}$ \\
\midrule
1   & 0.569022138974 \\
5   & 0.562006378463 \\
50  & 0.560563307692 \\
100 & 0.560482896107 \\
500 & 0.560418548434 \\
\bottomrule
\end{tabular}
\end{table}

\paragraph{Cross-check using a boundary-value formulation.}
To assess the numerical robustness of the reference calculation,
we additionally solve the same free-boundary problem using a
boundary-value problem (BVP) formulation. The unknown interval is
mapped to $t\in[0,1]$ by
\[
\pi=\varepsilon+(S-\varepsilon)t.
\]
The value function and threshold $S$ are determined simultaneously
using fourth-order collocation with residual control
\citep{KierzenkaShampine2001}.
We take $\varepsilon=0.002$, impose the regularity condition through
a tenth-order local expansion at the left endpoint, and impose
value matching and smooth fit at the right endpoint.
The residual tolerance is $10^{-9}$.
The BVP initialization does not use either the IVP threshold or
PINN predictions.

Table~\ref{tab:ou_mean_reference_check} shows that the two numerical
formulations yield stopping thresholds differing by at most
$3.05\times10^{-13}$ across all five dimensions.
Although both formulations share the same reduced model and
endpoint regularity condition, their agreement provides a numerical
consistency check. We use these mutually consistent thresholds
as numerical benchmarks for the PINN predictions.

\begin{table}[htbp]
\caption{Comparison of the IVP and BVP reference thresholds for
the tridiagonal configuration. Absolute differences are computed
before rounding the displayed thresholds.}
\label{tab:ou_mean_reference_check}
\centering
\begin{tabular}{@{}rrrr@{}}
\toprule
$n$
& $S_{\mathrm{IVP}}$
& $S_{\mathrm{BVP}}$
& $|S_{\mathrm{IVP}}-S_{\mathrm{BVP}}|$ \\
\midrule
1
& 0.569022138974
& 0.569022138973
& $3.0098\times10^{-13}$ \\

5
& 0.562006378463
& 0.562006378463
& $2.5302\times10^{-13}$ \\

50
& 0.560563307692
& 0.560563307692
& $2.6534\times10^{-13}$ \\

100
& 0.560482896107
& 0.560482896107
& $3.0465\times10^{-13}$ \\

500
& 0.560418548434
& 0.560418548433
& $2.8877\times10^{-13}$ \\
\bottomrule
\end{tabular}
\end{table}

\paragraph{Comparison with the PINN predictions.}
\label{benchmark_result}
Let
$
\overline{S}_{\mathrm{PINN}}
$
be the mean boundary prediction by PINN method.
We measure its absolute deviations from the two numerical benchmarks by
\[
E_{\mathrm{IVP},n}
=\left|\overline{S}_{\mathrm{PINN},n}-S_{\mathrm{IVP},n}\right|,
\qquad
E_{\mathrm{BVP},n}
=\left|\overline{S}_{\mathrm{PINN},n}-S_{\mathrm{BVP},n}\right|.
\]

To also assess pointwise discrepancies, we compute
\[
E_{\max,n}
=\max_{1\leq j\leq N_n}
\left|\widehat{S}_n(x_j)-S_{\mathrm{IVP},n}\right|.
\]
Because the reference boundary is constant, this maximum can be
computed directly from the sampled minimum and maximum predictions.

\begin{table}[htbp]
\caption{Comparison of PINN boundary predictions with
the numerical benchmarks}
\label{tab:ou_mean_pinn_benchmark}
\centering
\begin{tabular}{@{}rrrrr@{}}
\toprule
$n$
& $\overline{S}_{\mathrm{PINN}}$
& $E_{\mathrm{IVP}}$
& $E_{\mathrm{BVP}}$
& $E_{\max}$ \\
\midrule
5
& 0.562992
& $9.86\times10^{-4}$
& $9.86\times10^{-4}$
& $2.44\times10^{-3}$ \\

50
& 0.561457
& $8.94\times10^{-4}$
& $8.94\times10^{-4}$
& $2.42\times10^{-3}$ \\

100
& 0.561400
& $9.17\times10^{-4}$
& $9.17\times10^{-4}$
& $1.94\times10^{-3}$ \\

500
& 0.561265
& $8.46\times10^{-4}$
& $8.46\times10^{-4}$
& $2.15\times10^{-3}$ \\
\bottomrule
\end{tabular}
\end{table}

Table~\ref{tab:ou_mean_pinn_benchmark} shows close agreement
between the mean PINN predictions and numerical benchmarks.
The mean predictions are slightly above the reference thresholds,
with absolute deviations ranging from
$8.46\times10^{-4}$ to $9.86\times10^{-4}$.
These deviations remain below $1.0\times10^{-3}$ in all four
tested dimensions. The largest pointwise absolute error across
the sampled evaluation states is
$2.44\times10^{-3}$, remaining below $2.5\times10^{-3}$.
Thus, the benchmark comparison supports the accuracy of the learned
boundaries over the evaluated states.

\subsection{With Dimension = 1,000 in Example~\ref{exam:bm}}
\begin{figure}[H] 
\centering
\includegraphics[width=0.6\linewidth]{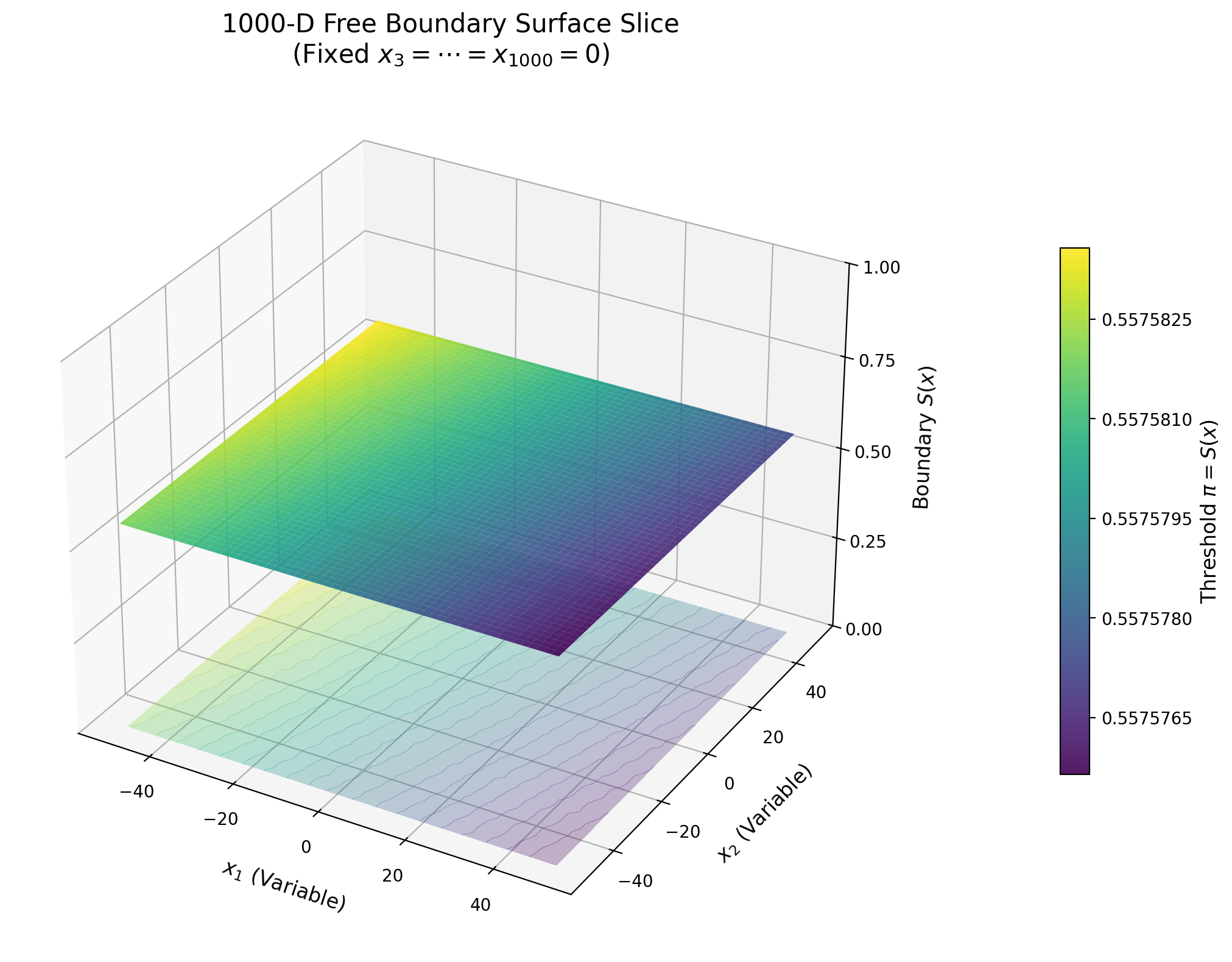}
\caption{With n=1,000 in Example~\ref{exam:bm}}
\label{app:WP_1000_D_boundary}
\end{figure}

Due to the space limitation, we will provide our result of Example~\ref{exam:bm} with dimension $n=1,000$ here.
Figure~\ref{app:WP_1000_D_boundary} shows the learned stopping boundary is approximately $0.557$, 
which is close to the theoretical result of $0.556$. 

\section{Implementation and Reproducibility Details}
\label{app:implementation}

\subsection{Computational resources and runtime}

All numerical experiments were conducted on a single workstation
running Windows 11. The system was equipped with an AMD Ryzen AI 9
HX 370 processor, 64 GB of system memory, and an NVIDIA GeForce
RTX 5090 Laptop GPU. Each training run used a single GPU, without
multi-GPU or distributed training.

The software environment consisted of Python 3.12.11,
PyTorch 2.8.0 with CUDA 12.8 support, and CUDA-enabled GPU execution.
For the experiments with dimensions up to $n=500$, each model was
trained for 50,000 iterations(80,000 for $n=1,000$). A typical training run required
approximately 30--50 minutes, depending on the dimension and the
specific model configuration. For experiments beyond 500 dimensions,
the number of training iterations was increased to at least 80,000.
The reported runtimes are approximate wall-clock times.

\subsection{Training hyperparameters}

The value function and stopping boundary are represented by two
separate fully connected neural networks. The value network
$V_\theta(\pi,x)$ takes the posterior probability $\pi$ together
with the $n$-dimensional state $x$ as input. It consists of three
hidden layers of width 128, each followed by a $\tanh$ activation,
and a scalar output layer. The boundary network $S_\phi(x)$ takes
only the state $x$ as input and consists of two hidden layers of
width 128 with $\tanh$ activations, followed by a scalar output layer.

For both networks, the spatial input is scaled according to
\[
x_{\mathrm{scaled}}
=
\frac{x}{x_{\max}-x_{\min}},
\]
while the posterior probability $\pi$ is used directly as an input
to the value network. 

The two networks are optimized jointly using Adam with separate
learning rates,
\[
\eta_V=10^{-3},
\qquad
\eta_S=2\times10^{-3}.
\]

The training batch size is set to 8,192. This choice reflects the
available computational resources rather than a deliberately optimized
batch size. In high-dimensional problems, larger collocation sets can
provide broader coverage of the computational domain, and the batch
size therefore represents a compromise between sampling coverage and
GPU memory requirements.

For the second-order differential terms in the PDE residual, we use
Hutchinson trace estimation with four probe vectors, following the
high-dimensional PINN formulation of \citet{HuEtAl2024HTE}.
This avoids explicit construction of the full Hessian and reduces the
computational and memory cost of evaluating high-dimensional
second-order operators.

\begin{table}[htbp]
\caption{Main training and implementation parameters used in the
numerical experiments.}
\label{tab:implementation_parameters}
\centering
\begin{tabular}{@{}ll@{}}
\toprule
Setting & Value \\
\midrule
Value network & $3\times128$, Tanh \\
Boundary network & $2\times128$, Tanh \\
Optimizer & Adam \\
Minimum learning rate & $10^{-6}$ \\
Iterations ($n\leq500$) & $50{,}000$ \\
Iterations ($n>500$) & $\geq80{,}000$ \\
Batch size & $8{,}192$ \\
Hutchinson probes & $4$ \\
Weight initialization & Xavier \\
\bottomrule
\end{tabular}
\end{table}

\subsection{Residual-based sampling}

Our adaptive sampling strategy is motivated by the residual-based
active sampling framework of \citet{gao2023active}, where collocation
points with larger residuals are sampled with higher probability.

We furthermore construct a composite residual score
$\eta(\pi,x)$ by combining the residual components appearing
in the loss function with boundary-related residuals specific
to our free-boundary problem.

At each adaptive-sampling step, we first generate 40,000 spatial
candidates using a mixture of local, radius-stratified, and
full-domain Sobol samples. The candidate pool is partitioned into
spatial metric bins. Let $\mathcal B_b$ denote the set of candidate
points in bin $b$, and let $\omega_b$ denote its prescribed sampling
weight. A corresponding quota of adaptive samples is first assigned
to each bin. Within a bin, candidate $i$ is selected according to
\[
p_i^{(b)}
=
\frac{\eta_i+\varepsilon}
{\sum_{j\in\mathcal B_b}(\eta_j+\varepsilon)},
\qquad
i\in\mathcal B_b,
\]
with $\varepsilon=10^{-8}$.
Thus, candidates with larger composite residuals are sampled more
frequently, while the bin-wise allocation prevents a small number
of high-residual far-field regions from dominating the adaptive
batch.

For each spatial candidate, the posterior probability is generated
from a mixture distribution: with probability $0.5$, it is sampled
near the current boundary,
\[
\pi=S_\phi(x)+\xi,
\qquad
\xi\sim\mathcal N(0,0.03^2),
\]
while the remaining samples provide broader posterior coverage.

The final training batch contains 8,192 points and combines
$25\%$ residual-adaptive samples, $25\%$ local samples,
$40\%$ radius-stratified shell samples, and $10\%$ full-domain
Sobol samples. This construction balances residual-driven
exploitation with continued exploration of the computational domain.

\begin{algorithm}[htbp]
\caption{Residual-based adaptive sampling}
\label{alg:residual_sampling}
\begin{algorithmic}[1]

\Require Current networks $V_\theta$, $S_\phi$;
batch size $N$; candidate pool size $N_{\mathrm{cand}}$;
spatial-bin weights $\{\omega_b\}$.

\State Generate a candidate pool
$\{x_i\}_{i=1}^{N_{\mathrm{cand}}}$

\State For each $x_i$, sample the posterior probability $\pi_i$:
with probability $0.5$,
\[
\pi_i=S_\phi(x_i)+\xi_i,
\qquad
\xi_i\sim\mathcal N(0,0.03^2),
\]
and otherwise sample $\pi_i$ from the broader posterior distribution.

\State Evaluate the composite residual
$\eta_i=\eta(\pi_i,x_i)$,
combining complementarity, region-specific,
value-matching, and boundary residual information.

\State Partition the candidate points into spatial metric bins
$\{\mathcal B_b\}$ and allocate an adaptive-sampling quota
to each bin according to $\omega_b$.

\For{each bin $\mathcal B_b$}
    \State Compute residual-based probabilities
    \[
    p_i^{(b)}
    =
    \frac{\eta_i+\varepsilon}
    {\sum_{j\in\mathcal B_b}(\eta_j+\varepsilon)},
    \qquad
    i\in\mathcal B_b.
    \]
    \State Sample the assigned number of adaptive points
    from $\mathcal B_b$ according to $p_i^{(b)}$.
\EndFor

\State Generate additional non-adaptive samples:
$25\%$ local, $40\%$ radius-stratified shell,
and $10\%$ full-domain Sobol samples.

\State Form the final batch using
$25\%$ residual-adaptive samples,
$25\%$ local samples,
$40\%$ shell samples,
and $10\%$ full-domain Sobol samples.

\State Shuffle the batch and return the resulting
$N=8192$ collocation points.

\end{algorithmic}
\end{algorithm}

\bibliography{references}

@article{CohenZhao2015,
  author={Cohen, Kobi and Zhao, Qing},
  title={Asymptotically Optimal Anomaly Detection via Sequential Testing},
  journal={IEEE Transactions on Signal Processing},
  volume={63},
  number={11},
  pages={2929--2941},
  year={2015},
}

@book{Liptser2001,
  author    = {R. S. Liptser and A. N. Shiryaev},
  title     = {Statistics of Random Processes I},
  edition   = {2nd},
  publisher = {Springer},
  year      = {2001},
  address   = {Berlin}
}

@book{Peskir2006,
  title={Optimal Stopping and Free Boundary Problems},
  author={Peskir, Goran and Shiryaev, Albert N},
  year={2006},
  publisher={Birkh{\"a}user}
}

@book{pham2009,
author    = {Huy\^{e}n Pham},
title     = {Continuous-Time Stochastic Control and Optimization with Financial Applications},
year      = {2009},
publisher = {Springer}
}

@article{Bayraktar2005,
  title={The standard Poisson disorder problem revisited},
  author={Bayraktar, Erhan and Dayanik, Savas and Karatzas, Ioannis},
  journal={Stochastic Processes and their Applications},
  volume={115},
  pages={1437-1450},
  year={2005}
}

@inproceedings{gal2016dropout,
  title     = {Dropout as a Bayesian Approximation: Representing Model Uncertainty in Deep Learning},
  author    = {Gal, Yarin and Ghahramani, Zoubin},
  booktitle = {Proceedings of the 33rd International Conference on Machine Learning},
  pages     = {1050--1059},
  year      = {2016},
  volume    = {48},
  series    = {Proceedings of Machine Learning Research}
}

@article{Beibel2000,
  title={A note on sequential detection with exponential penalty for the delay},
  author={Beibel, Michael},
  journal={Annals of Statistics},
  volume={28},
  pages={1696-1701},
  year={2000}
}

@inproceedings{lakshminarayanan2017simple,
  title     = {Simple and Scalable Predictive Uncertainty Estimation Using Deep Ensembles},
  author    = {Lakshminarayanan, Balaji and Pritzel, Alexander and Blundell, Charles},
  booktitle = {Advances in Neural Information Processing Systems},
  volume    = {30},
  pages     = {6402--6413},
  year      = {2017}
}

@article{Gapeev2006,
  title={The Wiener disorder problem with finite horizon},
  author={Gapeev, Pavel V. and Peskir, Goran},
  journal={Stochastic Processes and their Applications},
  volume={116},
  pages={1770-1791},
  year={2006}
}

@article{Gapeev2013,
title={Bayesian quickest detection problems for some diffusion processes},
author={Gapeev, Pavel V. and Shiryaev, Albert N.},
journal={Advances in Applied Probability},
volume={45},
number={1},
pages={164-185},
year={2013},
month={March}
}

@article{Gatta2023,
author = {Gatta, Federico and Schiano Di Cola, Vincenzo and Giampaolo, Fabio and Piccialli, Francesco and Cuomo, Salvatore},
title = {Meshless methods for American option pricing through Physics-Informed Neural Networks},
journal={Engineering Analysis with Boundary Elements},
volume = {151},
year = {2023},
pages = {68-82}
}

@article{Leevy2020,
  author = {Leevy, Joffrey L. and Khoshgoftaar, Taghi M.},
  title = {A survey and analysis of intrusion detection models
           based on {CSE-CIC-IDS2018} Big Data},
  journal = {Journal of Big Data},
  volume = {7},
  pages = {104},
  year = {2020},
  doi = {10.1186/s40537-020-00382-x}
}

@article{Poor1998,
  title={Quickest detection with exponential penalty for delay},
  author={Poor, H. Vincent},
  journal={Annals of Statistics},
  volume={26},
  pages={2179-2205},
  year={1998}
}

@article{Raissi2019,
author = {Raissi, M. and Perdikaris, P. and Karniadakis, G. E.},
title = {Physics-informed neural networks: A deep learning framework for solving forward and inverse problems involving nonlinear partial differential equations},
journal = {Journal of Computational Physics},
volume = {378},
pages = {686-707},
year = {2019}
}

@article{Raissi2020,
title={Hidden fluid mechanics: Learning velocity and pressure fields from flow visualizations},
author={Raissi, Maziar and Yazdani, Alireza and Karniadakis, George Em},
journal={Science},
volume={367},
number={6481},
pages={1026-1030},
year={2020},
publisher={American Association for the Advancement of Science}
}

@article{Shiryaev1963,
author= {A. N. Shiryaev},
title= {On optimum methods in quickest detection problems},
journal= {Theory of Probability and its Applications},
  volume    = {8},
  number    = {1},
  pages     = {22-46},
  year      = {1963}
}

@article{Shiryaev1964,
  title={On Markov sufficient statistics in non-additive Bayes problems of sequential analysis},
  author={Shiryaev, Albert N.},
  journal={Theory of Probability and its Applications},
  volume={9},
  pages={670-686},
  year={1964}
}

@article{Shiryaev1965,
  title={Some exact formulas in a 'disorder' problem},
  author={Shiryaev, Albert N.},
  journal={Theory of Probability and its Applications},
  volume={10},
  pages={348-354},
  year={1965}
}

@article{Shiryaev1967,
  author    = {A. N. Shiryaev},
  title     = {Two problems of sequential analysis},
  journal   = {Cybernetics},
  volume    = {3},
  pages     = {63-69},
  year      = {1967}
}

@inproceedings{tasoulis2018,
  author = {Tasoulis, Sotiris K. and Vrahatis, Aristidis G. and
            Georgakopoulos, Spiros V. and Plagianakos, Vassilis P.},
  title = {Real Time Sentiment Change Detection of {Twitter} Data Streams},
  booktitle = {2018 Innovations in Intelligent Systems and
               Applications (INISTA)},
  pages = {1--6},
  year = {2018},
  publisher = {IEEE},
  doi = {10.1109/INISTA.2018.8466326}
}

@article{Xie2021,
  author = {Xie, Liyan and Zou, Shaofeng and Xie, Yao and
            Veeravalli, Venugopal V.},
  title = {Sequential (Quickest) Change Detection:
           Classical Results and New Directions},
  journal = {IEEE Journal on Selected Areas in Information Theory},
  volume = {2},
  number = {2},
  pages = {494--514},
  year = {2021},
  doi = {10.1109/JSAIT.2021.3072962}
}

@incollection{Kolmogorov1992,
  author    = {Kolmogorov, A. N.},
  title     = {On Analytical Methods in Probability Theory},
  booktitle = {Selected Works of A. N. Kolmogorov, vol. II},
  editor    = {Shiryaev, A. N.},
  year      = {1992},
  pages     = {62-108},
  publisher = {Kluwer},
  address   = {Dordrecht},
}

@article{gao2023active,
  title={Active learning based sampling for high-dimensional nonlinear partial differential equations},
  author={Gao, Wenhan and Wang, Chunmei},
  journal={Journal of Computational Physics},
  volume={475},
  pages={111848},
  year={2023},
  publisher={Elsevier}
}

@article{becker2019deep,
  title={Deep optimal stopping},
  author={Becker, Sebastian and Cheridito, Patrick and Jentzen, Arnulf},
  journal={Journal of Machine Learning Research},
  volume={20},
  number={74},
  pages={1--25},
  year={2019}
}

@article{herrera2024optimal,
  title={Optimal stopping via randomized neural networks},
  author={Herrera, Calypso and Krach, Florian and Ruyssen, Pierre and Teichmann, Josef},
  journal={Frontiers of Mathematical Finance},
  volume={3},
  number={1},
  pages={31--77},
  year={2024}
}

@article{Wang2021Stefan,
  title={Deep learning of free boundary and Stefan problems},
  author={Wang, Sifan and Perdikaris, Paris},
  journal={Journal of Computational Physics},
  volume={428},
  pages={109914},
  year={2021},
  publisher={Elsevier}
}

@article{HuEtAl2024HTE,
  author  = {Hu, Zheyuan and Shi, Zekun and
             Karniadakis, George Em and Kawaguchi, Kenji},
  title   = {Hutchinson Trace Estimation for High-Dimensional
             and High-Order Physics-Informed Neural Networks},
  journal = {Computer Methods in Applied Mechanics and Engineering},
  volume  = {424},
  pages   = {116883},
  year    = {2024},
  doi     = {10.1016/j.cma.2024.116883}
}

@article{GloverPeskir2024,
  author  = {Glover, Kristoffer and Peskir, Goran},
  title   = {Quickest Detection Problems for
             {Ornstein--Uhlenbeck} Processes},
  journal = {Mathematics of Operations Research},
  year    = {2024},
  volume  = {49},
  number  = {2},
  pages   = {1045--1064},
  doi     = {10.1287/moor.2021.0186}
}

@book{HairerWanner1996,
  author    = {Hairer, Ernst and Wanner, Gerhard},
  title     = {Solving Ordinary Differential Equations {II}:
               Stiff and Differential-Algebraic Problems},
  publisher = {Springer},
  year      = {1996},
  edition   = {2},
  series    = {Springer Series in Computational Mathematics},
  volume    = {14},
  doi       = {10.1007/978-3-642-05221-7}
}

@article{KierzenkaShampine2001,
  author  = {Kierzenka, Jacek and Shampine, Lawrence F.},
  title   = {A {BVP} Solver Based on Residual Control
             and the {MATLAB} {PSE}},
  journal = {ACM Transactions on Mathematical Software},
  year    = {2001},
  volume  = {27},
  number  = {3},
  pages   = {299--316},
  doi     = {10.1145/502800.502801}
}

\end{document}